\documentclass[twoside,11pt]{article}

\usepackage[utf8]{inputenc}
\usepackage[T1]{fontenc}
\usepackage{amsmath}
\usepackage{amsthm}
\let\DOMDproof\proof
\let\endDOMDproof\endproof

\usepackage[preprint,nohyperref]{jmlr2e}

\let\proof\DOMDproof
\let\endproof\endDOMDproof

\usepackage{mathtools}
\usepackage{microtype}
\usepackage{xcolor}
\usepackage{booktabs}
\usepackage{multirow}
\usepackage{tabularx}
\usepackage{makecell}
\usepackage{enumitem}
\usepackage{algorithm}
\usepackage[noend]{algpseudocode}
\usepackage{lastpage}

\usepackage{hyperref}
\usepackage[capitalize,noabbrev]{cleveref}
\usepackage{abbreviation_definitions}

\usepackage{minitoc}
\setcitestyle{authoryear,round,citesep={;},aysep={,},yysep={;}}
\hypersetup{
    hidelinks,
    pdftitle={Nonstationary Generalized Linear Bandits with Discounted Online Mirror Descent},
    pdfauthor={Joongkyu Lee and Min-hwan Oh}
}

\ShortHeadings
    {Nonstationary GLBs with Discounted OMD}
    {Lee and Oh}
\firstpageno{1}

\editor{}

\newcommand{\includeprojectfigure}[2][]{%
    \IfFileExists{#2}{%
        \includegraphics[#1]{#2}%
    }{%
        \PackageWarningNoLine{DOMD-GLB}{Missing figure: #2. Placeholder printed}%
        \fbox{\parbox[c][22mm][c]{\dimexpr\linewidth-2\fboxsep-2\fboxrule\relax}{%
            \centering\normalfont
            Figure file not supplied\\[1ex]
            \texttt{\detokenize{#2}}\\[1ex]
            \small Restore the original PDF to display this figure.
        }}%
    }%
}

\title{Nonstationary Generalized Linear Bandits\\
with Discounted Online Mirror Descent}

\author{%
    \name Joongkyu Lee \email jl7635@columbia.edu \\
    \addr Columbia University
    \AND
    \name Min-hwan Oh \email minoh@snu.ac.kr \\
    \addr Seoul National University
}

\begin{document}
\doparttoc
\faketableofcontents
\maketitle
\thispagestyle{plain}

\begin{abstract}
We study nonstationary generalized linear bandits, where the expected reward is modeled through a nonlinear link function with an unknown \textit{time-varying} parameter.
This framework encompasses a broad class of reward models, including linear, Bernoulli, and binomial rewards.
Existing approaches are predominantly based on maximum-likelihood estimation (MLE), using sliding-window, restart, or discounting mechanisms to handle nonstationarity.
Although these methods achieve statistically efficient regret guarantees, they generally require revisiting past observations at every round, which leads to computation and memory costs that grow with time; moreover, several of them rely on a non-convex projection step.
In this paper, we propose \texttt{DOMD-GLB}, a new algorithm for nonstationary GLBs that utilizes \textit{discounted online mirror descent (DOMD)} for parameter estimation, thereby
incurring only $\mathcal{O}(1)$ computation and memory costs per round.
We prove dynamic regret bounds of order $\widetilde{\mathcal{O}} \big(c_\mu^{-1/2} d^{3/4} P_T^{1/4} T^{3/4}\big)$ in drifting environments
and $\widetilde{\mathcal{O}}\big(c_\mu^{-1/3} d^{2/3} \Gamma_T^{1/3} T^{2/3}\big) $in piecewise-stationary environments,
where $d$ denotes the feature dimension, $T$ the time horizon, $P_T$ the path length, $\Gamma_T$ the number of change points, and $c_\mu$ a curvature parameter associated with the link function,
while substantially improving computational efficiency over prior work.
To the best of our knowledge, this is the first algorithm for nonstationary GLBs with per-round computation and memory costs independent of time.
\end{abstract}

\section{Introduction}
\label{sec:Introduction}
Nonstationary parametric bandits~\citep{cheung2019learning,russac2019weighted,zhao2020simple,russac2020algorithms,faury2021regret,russac2021self,wei2021non,deng2022weighted,liu2023nonstationary,wang2023revisiting,hong2023optimization,jia2023smooth,wang2024variance,li2025constrained} model sequential decision-making problems in which the reward distribution is governed by an unknown \emph{time-varying} parameter.
This framework has attracted substantial attention in recent years, due to its relevance to a wide range of online applications in evolving environments, including recommendation systems~\citep{tomkins2008intelligentpooling,huleihel2021learning}, dynamic pricing~\citep{zhao2023high}, and related settings with changing user preferences or market conditions.

In this paper, we study nonstationary generalized linear bandits (GLBs)~\citep{russac2020algorithms,faury2021regret,russac2021self,wei2021non,wang2023revisiting}
with \textit{time-varying contextual action sets}.
At each round \(t\), the learner observes a feasible action set
\(\mathcal X_t \subseteq \RR^d\), which may vary arbitrarily over time,
selects an action \(x_t \in \mathcal X_t\), and receives a reward satisfying
$\EE[r_t \mid x_t] = \mu(x_t^\top \theta_t^\star).$
Here, \(\mu : \RR \to \RR\) is a known nonlinear link function determined by the
underlying generalized linear model, and
\(\theta_t^\star \in \RR^d\) is an unknown \textit{time-varying} parameter.
Thus, both the available contextual actions \(\mathcal X_t\) and the underlying
reward parameter \(\theta_t^\star\) may change over time.
To quantify parameter nonstationarity, we consider two standard measures:
(i) the path length
\(P_T := \sum_{t=1}^{T-1}\|\theta_{t+1}^\star-\theta_t^\star\|_2\)
in drifting environments, and
(ii) the number of change points
\(\Gamma_T := \sum_{t=1}^{T-1}
\mathbbm{1}\{\theta_{t+1}^\star\neq\theta_t^\star\}\)

\begin{table}[t]
\centering
\caption{
Comparison of dynamic regret bounds for nonstationary GLM bandits.
Here, \(d\) is the feature dimension, \(1/c_{\mu}\) measures nonlinearity,
and \(P_T\) and \(\Gamma_T\) denote the path length and number of change points.
For computation and memory, we report only the dependence on \(t\).
``SC'' denotes self-concordance, and ``No poly-time guar.'' means that no
polynomial-time guarantee is available.
``Varying'' allows time-varying feasible action sets; ``Fixed'' assumes a fixed
action set; and ``Fixed\(^{\ast}\)'' indicates that the stated result is for a
fixed action set but its analysis appears to extend directly to varying action sets.
\(^{\ddagger}\)The original \texttt{RestartGLB} analysis contains a technical flaw.
Subsequent work~\citep{zhao2021non} corrects the underlying linear-bandit
analysis and obtains
\(\widetilde O(d^{7/8}P_T^{1/4}T^{3/4})\),
but does not provide a corrected GLB guarantee.
}
\resizebox{\textwidth}{!}{
\begin{tabular}{clccccc}
\toprule
\multicolumn{1}{l}{}
& Paper
& Algorithm Type
& Assumption
& Regret
& Comput. per Round
& Memory
\\
\midrule

\multirow{11}{*}{\begin{tabular}[c]{@{}c@{}}Drifting\end{tabular}}

& \shortstack[l]{\citet{zhao2020simple}\\
  \textcolor{gray}{\texttt{(RestartGLB)}}}
& Restart + MLE
& Sub-Gaussian Noise + Fixed$^\ast$
& $\BigOTilde\!\left(
    c_\mu^{-1}
    d^{2/3} P_T^{1/3}T^{2/3}
  \right)^{\ddagger}$
& \textcolor{red}{$\BigO(T^{2/3})$}
& \textcolor{red}{$\BigO(T^{2/3})$}
\\

& \shortstack[l]{\citet{faury2021regret}\\
  \textcolor{gray}{\texttt{(BVD-GLM-UCB)}}}
& Discount + MLE
& Bounded Reward + Varying
& $\BigOTilde\!\left(
    c_\mu^{-1}
    d^{9/10}P_T^{1/5}T^{4/5}
  \right)$
& \textcolor{red}{No poly-time guar.}
& \textcolor{red}{$\BigO(t)$}
\\

& \shortstack[l]{\citet{wei2021non}\\
  \textcolor{gray}{\texttt{(MASTER \!\!+\!\! GLM-UCB)}}}
& Restart + MLE
& Bounded Reward + Fixed
& $\BigOTilde\!\left(
    c_\mu^{-1}
    dP_T^{1/3}T^{2/3}
  \right)$
& \textcolor{red}{No poly-time guar.}
& \textcolor{red}{$\BigO(t\log t)$}
\\

& \shortstack[l]{\citet{wang2023revisiting}\\
  \textcolor{gray}{\texttt{(GLB-WeightUCB)}}}
& Discount + MLE
& Sub-Gaussian Noise + Fixed$^\ast$
& $\BigOTilde\!\left(
    c_\mu^{-3/4}
    d^{3/4}P_T^{1/4}T^{3/4}
  \right)$
& \textcolor{red}{No poly-time guar.}
& \textcolor{red}{$\BigO(t)$}
\\

& \shortstack[l]{\citet{wang2023revisiting}\\
  \textcolor{gray}{\texttt{(SCB-WeightUCB)}}}
& Discount + MLE
& Bounded Reward + SC + Fixed$^\ast$
& $\BigOTilde\!\left(
    c_\mu^{-1/2}
    d^{3/4}P_T^{1/4}T^{3/4}
  \right)$
& \textcolor{red}{No poly-time guar.}
& \textcolor{red}{$\BigO(t)$}
\\

& \shortstack[l]{\textbf{This Work}\\
  \textcolor{gray}{(\AlgName{}, Theorem~\ref{thm:regret})}}
& Discount + \textcolor{blue}{OMD}
& Bounded Reward + \textcolor{blue}{Varying}
& $\BigOTilde\!\left(
    c_\mu^{-1/2}
    d^{3/4}P_T^{1/4}T^{3/4}
  \right)$
& \textcolor{blue}{$\BigO(1)$}
& \textcolor{blue}{$\BigO(1)$}
\\

\midrule

\multirow{8}{*}{\begin{tabular}[c]{@{}c@{}}Piecewise\\Stationary\end{tabular}}

& \shortstack[l]{\citet{russac2020algorithms}\\
  \textcolor{gray}{\texttt{(D-GLUCB)}}}
& Discount + MLE
& Bounded Reward + Varying
& $\BigOTilde\!\left(
    c_\mu^{-1}
    d^{2/3}\Gamma_T^{1/3}T^{2/3}
  \right)$
& \textcolor{red}{No poly-time guar.}
& \textcolor{red}{$\BigO(t)$}
\\

& \shortstack[l]{\citet{wei2021non}\\
  \textcolor{gray}{\texttt{(MASTER \!\!+\!\! GLM-UCB)}}}
& Restart + MLE
& Bounded Reward + Fixed
& $\BigOTilde\!\left(
    c_\mu^{-1}
    d\sqrt{\Gamma_TT}
  \right)$
& \textcolor{red}{No poly-time guar.}
& \textcolor{red}{$\BigO(t\log t)$}
\\

& \shortstack[l]{\citet{russac2021self}\\
  \textcolor{gray}{\texttt{(SC-D-GLUCB)}}}
& Discount + MLE
& Bounded Reward + SC + Varying
& $\BigOTilde\!\left(
    c_\mu^{-1/3}
    d^{2/3}\Gamma_T^{1/3}T^{2/3}
  \right)$
& \textcolor{red}{$\BigO(t)$}
& \textcolor{red}{$\BigO(t)$}
\\

& \shortstack[l]{\citet{wang2023revisiting}\\
  \textcolor{gray}{\texttt{(SCB-PW-WeightUCB)}}}
& Discount + MLE
& Bounded Reward + SC + Fixed$^\ast$
& $\BigOTilde\!\left(
    d^{2/3}\Gamma_T^{1/3}T^{2/3}
  \right)$
& \textcolor{red}{No poly-time guar.}
& \textcolor{red}{$\BigO(t)$}
\\

& \shortstack[l]{\textbf{This Work}\\
  \textcolor{gray}{(\AlgName{}, Theorem~\ref{thm:regret_piecewise})}}
& Discount + \textcolor{blue}{OMD}
& Bounded Reward + \textcolor{blue}{Varying}
& $\BigOTilde\!\left(
    c_\mu^{-1/3}
    d^{2/3}\Gamma_T^{1/3}T^{2/3}
  \right)$
& \textcolor{blue}{$\BigO(1)$}
& \textcolor{blue}{$\BigO(1)$}
\\

\bottomrule
\end{tabular}
}
\label{tab:regrets}
\end{table}

To handle nonstationarity, existing GLB methods rely predominantly on
MLE-based approaches, including restart~\citep{zhao2020simple,wei2021non}
and weighted (or discounted)~\citep{russac2020algorithms,faury2021regret,
russac2021self,wang2023revisiting} strategies
(see Table~\ref{tab:regrets}).
Notably, the GLB guarantee of \texttt{MASTER}~\citep{wei2021non} is established
for a \textit{fixed} action set and does not directly extend to the\textit{ time-varying action}
sets considered in our setting.
Beyond this difference in problem scope, existing MLE-based methods also
incur substantial computational and memory costs, since their estimators must
be repeatedly recomputed from retained observations.
Moreover, several prior works~\citep{faury2021regret,russac2020algorithms,
wang2023revisiting} rely on a \textit{non-convex projection step} of the form
$\widetilde{\theta}_t
\in
\argmin_{\theta}
\big\|
\sum_s
\bigl(
\mu(x_s^\top \theta)-\mu(x_s^\top \widehat{\theta}_t)
\bigr)x_s
\big\|_{V_t^{-1}},$
where \(V_t \in \RR^{d\times d}\) is positive definite.
This projection is generally difficult to compute in polynomial time.

On the other hand, recent work on \textit{stationary} GLBs~\citep{zhang2025generalized} showed that online mirror descent (OMD) can estimate the model parameter in a \textit{one-pass} manner, without storing past data.
This yields constant per-round computation and memory costs while retaining confidence guarantees comparable to those of statistically efficient MLE-based methods, such as~\citet{lee2024unified}.
This suggests OMD as a natural starting point for developing computationally efficient algorithms for nonstationary GLBs.
This leads to the following question:
\begin{center}
    \textit{Can we design a nonstationary GLB algorithm with constant computation \\and memory costs?}
\end{center}
In this paper, we answer this question affirmatively by proposing \AlgName{}, a discounted OMD (DOMD)-based algorithm that achieves both statistical and computational efficiency.

\paragraph{Technical difficulties.}
The stationary OMD analysis of~\citet{zhang2025generalized} does not directly
extend to the discounted setting, and several key components of the analysis
must be redesigned.
The first challenge lies in the \textit{mirror-descent geometry}.
In stationary GLBs~\citep{zhang2025generalized}, the additively accumulated
Hessian naturally induces the geometry underlying the OMD update.
In the nonstationary setting, however, this geometry must incorporate
forgetting: stale curvature information should be progressively downweighted
rather than accumulated uniformly.
Consequently, the stationary geometry is no longer compatible with the
discounted recursion, and it is not clear \emph{a priori} what the appropriate
discounted analogue of the stationary Bregman divergence should be.
Constructing this geometry is therefore a central difficulty in extending OMD
to nonstationary GLBs.

The second challenge lies in the \textit{concentration analysis}.
The confidence analysis of~\citet{zhang2025generalized} is based on a \textit{mix-loss} argument~\citep{vovk2001competitive}, which constructs a suitable exponentiated stochastic process and shows that it forms a supermartingale, thereby allowing the use of \textit{Ville's inequality}~\citep{ville1939etude} to derive an anytime concentration bound.
In the discounted setting, however, the corresponding discounted process is generally not a supermartingale, and this argument therefore breaks down.
As a result, a different concentration analysis is required (see Remark~\ref{remark:why_fixed}).

The third challenge is to control the \textit{loss drift} induced by nonstationarity.
In the DOMD confidence analysis, the reference parameter is no longer fixed over time, since the current parameter may differ from the parameters underlying past observations.
This mismatch creates an additional \textit{loss-drift term} that is absent in the stationary setting.
The main difficulty is therefore to isolate and control this term within the DOMD confidence argument.
Unlike prior nonstationary MLE-based analyses (e.g.,~\citealt{wang2023revisiting}), where the effect of parameter drift can be isolated as an explicit additive \textit{parameter-bias term}, DOMD requires controlling the drift in \textit{loss space}, through terms of the form \(\ell_s(\theta_t^\star)-\ell_s(\theta_s^\star)\).
Thus, controlling loss drift requires a separate analysis tailored to the DOMD estimator.

\paragraph{Contributions.}
To address these difficulties, we develop a proper discounted mirror-descent geometry (Subsection~\ref{subsec:DOMD}),
a new mix-loss-based analysis tailored to DOMD (Lemma~\ref{lemma:weighted-exp}),
and a separate loss-drift bound in terms of the nonstationarity measure (Lemma~\ref{lemma:drift-loss}).
Together, these yield a high-probability confidence bound for the DOMD estimator (Theorem~\ref{thm:DOMD_confidence}), which in turn leads to dynamic regret guarantees in both drifting (Theorem~\ref{thm:regret}) and piecewise-stationary environments (Theorem~\ref{thm:regret_piecewise}).
Our contributions are summarized as follows:
\begin{itemize}
    \item We propose a DOMD method for parameter estimation in nonstationary generalized linear bandits (GLBs). The resulting update requires only \(\BigO(d^3)\) computation and \(\BigO(d^2)\) memory per round, both independent of the time index \(t\) (Subsection~\ref{subsec:DOMD}).

    \item We establish a new confidence bound for DOMD in nonstationary GLBs by extending the \textit{mix-loss} method~\citep{vovk2001competitive} to the discounted recursion and by explicitly controlling the drift terms that arise from nonstationarity (Theorem~\ref{thm:DOMD_confidence}).

    \item In the drifting setting, we prove the dynamic regret bound \(\BigOTilde\big(c_\mu^{-1/2} d^{3/4} P_T^{1/4} T^{3/4}\big)\), matching the best-known regret order among discounted MLE-based methods~\citep{wang2023revisiting} while using a substantially more efficient update (Theorem~\ref{thm:regret}).

    \item In the piecewise-stationary setting, we obtain the regret bound \(\BigOTilde\big(c_\mu^{-1/3} d^{2/3} \Gamma_T^{1/3} T^{2/3}\big)\), showing that the DOMD framework also provides a strong statistical guarantee when nonstationarity is measured by the number of parameter changes (Theorem~\ref{thm:regret_piecewise}).

\end{itemize}

\section{Related Works}
\label{sec:related}
\paragraph{Nonstationary parametric bandits.}
Nonstationary parametric bandits have been studied extensively across linear and generalized linear models
\citep{cheung2019learning,russac2019weighted,zhao2020simple,russac2020algorithms,faury2021regret,russac2021self,wei2021non,deng2022weighted,liu2023nonstationary,hong2023optimization,jia2023smooth,wang2023revisiting,wang2024variance,li2025constrained}.
In the nonstationary GLB setting most relevant to our work, existing methods rely on MLE-based estimators, adapting them to nonstationarity through windowing, discounting, or restarting
\citep{russac2020algorithms,faury2021regret,russac2021self,wang2023revisiting}.
Among restart-based approaches, \texttt{MASTER}~\citep{wei2021non} provides an adaptive black-box framework that detects nonstationarity and restarts the base algorithm, achieving the optimal dynamic regret
\(\BigOTilde(\min\{\sqrt{\Gamma_T T},\, P_T^{1/3}T^{2/3}\})\).
Although these methods enjoy strong statistical guarantees, they typically require repeated re-optimization over past data, leading to computation and memory costs that grow with time, the window size, or the restart period.
Moreover, some also rely on a non-convex projection step that is generally not computable in polynomial time.

\paragraph{Stationary generalized linear bandits.}
The stationary GLB literature has long faced a trade-off between statistical sharpness and computational efficiency.
The pioneering optimism-based algorithm of~\citet{filippi2010parametric} requires repeated maximum likelihood
estimation (MLE) over the full history,
and its regret scales with the global slope parameter \(1/c_\mu\), which can be exponentially large in the parameter-norm bound.
Subsequent work has pursued two complementary directions: statistical sharpness and computational efficiency.
On the statistical side, recent MLE-based analyses use self-concordance and local curvature to replace this global dependence with sharper instance-dependent quantities, yielding nearly optimal regret bounds in terms of local parameters such as \(\kappa^\star\)~\citep{lee2024unified,liu2024almost,sawarni2024generalized}.
On the computational side, efficient methods such as \texttt{GLOC}~\citep{jun2017glm} use online updates and achieve constant per-round cost, but their regret bounds still depend linearly on \(1/c_\mu\).
More recently, \citet{zhang2025generalized} showed that these two goals can be achieved \textit{simultaneously}, obtaining nearly optimal regret while maintaining fully online updates with constant time and memory per round.

\paragraph{OMD in bandits and reinforcement learning.}
Online mirror descent (OMD) has recently emerged as a useful tool for designing jointly efficient bandit algorithms, replacing repeated batch MLE by one-pass updates whose per-round cost remains independent of time.
In logistic bandits, \citet{faury2022jointly,zhang2024online} developed computationally efficient UCB-type algorithms based on OMD.
This approach has since been extended to broader settings, including multinomial logit (MNL) bandits \citep{lee2024nearly,lee2025improved} and GLBs \citep{zhang2025generalized}.
Related OMD-style ideas have also been explored in other online decision-making problems, including heavy-tailed linear bandits \citep{wang2025heavy}, 
cascading bandits \citep{choi2025true}, 
reinforcement learning \citep{cho2024randomized,lee2025combinatorial}, and 
RLHF/PbRL \citep{li2025provably,lee2025preference}.
However, to the best of our knowledge, the use of OMD in nonstationary settings remains largely unexplored.

\section{Preliminaries}
\label{sec:preliminary}
\paragraph{Notations.}
For \(M\succeq 0\) and \(x\in\mathbb{R}^d\), let
\(\|x\|_M:=\sqrt{x^\top Mx}\) and \(\|x\|_2:=\sqrt{x^\top x}\).
We write \([T]:=\{1,\ldots,T\}\), use \(A\succeq B\) to mean that \(A-B\) is positive semidefinite, and denote the first and second derivatives of \(f:\mathbb{R}\to\mathbb{R}\) by \(f'\) and \(f''\), respectively.

\subsection{Nonstationary Generalized Linear Bandit}
\label{subsec:problem_setting}
We consider a $T$-round contextual bandit problem.
At each round $t \in [T]$, the learner observes a feasible action set $\Xcal_t \subset \mathbb{R}^d$, chooses an action $x_t \in \Xcal_t$, and then receives a stochastic reward $r_t \in \mathbb{R}$. The environment is governed by an unknown time-varying parameter sequence $\{ \theta_t^\star \}_{t=1}^{\infty}$.
Conditioned on the chosen action \(x_t\), the reward follows a canonical exponential-family distribution whose natural parameter is given by the linear predictor \(x_t^\top \theta_t^\star\), i.e.,
$    \text{Pr} \left[r_t | x_t,\theta_t^\star \right]
    =
    \exp\!\left(
    \frac{r_t x_t^\top \theta_t^\star - m(x_t^\top \theta_t^\star)}{g(\tau)}
    +
    h(r_t,\tau)
    \right),$
where $m:\mathbb{R}\to\mathbb{R}$ is convex and three times continuously differentiable, $g(\tau)>0$ is a known dispersion term that controls the variability of the distribution, and $h(\cdot,\tau)$ is the base measure.
Let
$\mu(z) := m'(z)$.
Then the conditional mean satisfies
$\mathbb{E}[r_t | x_t,\theta_t^\star] = \mu(x_t^\top \theta_t^\star)$ (Proposition 3.1 of~\citealt{wainwright2008graphical}).
We assume that the parameter path $\{\theta_t^\star\}_{t\ge 1}$ is fixed independently of the reward noise process, following the standard exogenous-drift formulation in the nonstationary bandit literature~\citep{cheung2019learning,russac2019weighted,russac2021self,zhao2020simple,wang2023revisiting}.

We make the following standard assumptions used in GLBs~\citep{filippi2010parametric}:
\begin{assumption}[Boundedness]
\label{assum:boundedness}
    For all $t \ge 1$, we assume $\theta_t^\star \in \Theta := \{\theta \in \mathbb{R}^d : \|\theta\|_2 \le S\}$,
    $\|x\|_2 \le 1$ for all $x \in \mathcal{X}_t$,
    and
    $r_t \in [0, R]$ almost surely for some known $R \in \RR$.
\end{assumption}
\begin{assumption}[Bounded link function]
\label{assum:bounded_link}
    The link function $\mu$ is twice differentiable over its feasible domain. Moreover, there exist constants $c_\mu, k_{\mu} >0$ such that $c_\mu \le \mu'(z) \le k_\mu$ for all $z \in [-S, S]$.
    Consequently, the function $m$ is strictly convex and $\mu$ is strictly increasing.
\end{assumption}
\begin{proposition} [Self-concordance, Lemma 2.2 of~\citealt{sawarni2024generalized}]
\label{prop:self-concordance}
    For any GLM supported on $[0, R]$ almost surely, the link function $\mu(\cdot)$ satisfies $|\mu''(z)| \le R \cdot \mu'(z)$ for all $z \in \RR$.
\end{proposition}
Therefore, under the bounded rewards in Assumption~\ref{assum:boundedness}, self-concordance is automatically satisfied.
As a result, in previous works such as~\citep{russac2021self,wang2023revisiting}, the self-concordance assumption is in fact redundant, since bounded rewards are already imposed.
More importantly, this induced self-concordant property is essential for obtaining regret bounds with improved dependence on \(1/c_\mu\), which can be exponentially large in the parameter-norm bound, i.e., \(1/c_\mu = \BigO(e^S)\).

In this paper, we consider two standard forms of nonstationarity.

\paragraph{Setting 1: Drifting environment.}
The environment exhibits gradual drift over time, and nonstationarity is measured by the path length
$P_T := \sum_{t=1}^{T-1} \|\theta_{t+1}^\star - \theta_t^\star\|_2,$
which quantifies the cumulative variation of the underlying parameter sequence.

\paragraph{Setting 2: Piecewise-stationary environment.}
The environment is piecewise stationary over time, and nonstationarity is measured by the number of parameter changes
$\Gamma_T := \sum_{t=1}^{T-1} \mathbbm{1}\!\left\{\theta_{t+1}^\star \neq \theta_t^\star\right\},$
which counts how many times the underlying parameter changes over the horizon.

Our goal is to minimize the cumulative \textit{dynamic (pseudo) regret}:
\begin{align*}
    \Regret := \sum_{t=1}^T
    \left(
    \mu((x_t^\star)^\top \theta_t^\star)
    -
    \mu(x_t^\top \theta_t^\star)
    \right),
    \qquad 
    \text{where}\quad
    x_t^\star \in \argmax_{x\in \Xcal_t} \mu(x^\top \theta_t^\star)
        =
        \argmax_{x\in \Xcal_t} x^\top \theta_t^\star.
\end{align*}
while keeping the per-round computational cost constant, independently of $t$.

\section{Algorithm and Main Results}
\label{sec:algorithm}
In this section, we first introduce a DOMD method (Subsection~\ref{subsec:DOMD}), then present the OMD-based nonstationary GLB algorithm \AlgName{} (Subsection~\ref{subsec:action_selection}), and finally establish dynamic regret guarantees for both drifting (Subsection~\ref{subsec:regret_drift}) and piecewise-stationary (Subsection~\ref{subsec:piecewise_regret}) settings.
\subsection{Discounted Online Mirror Descent}
\label{subsec:DOMD}
We estimate the current parameter via \emph{discounted online mirror descent} (DOMD).
For each round \(t\), we define the loss function as the negative log-likelihood
$\ell_t(\theta)
:=
\left( m(x_t^\top \theta)-r_t x_t^\top \theta \right)/g(\tau).$
Given the current iterate \(\theta_t\), we use the local quadratic surrogate
\[
\widetilde{\ell}_t(\theta)
:=
\ell_t(\theta_t)
+
\langle \nabla \ell_t(\theta_t),\, \theta-\theta_t\rangle
+
\frac12 \|\theta-\theta_t\|_{\nabla^2 \ell_t(\theta_t)}^2.
\]
Let
$
H_t
=
\lambda I_d+\sum_{s=1}^{t-1}\gamma^{t-1-s}\nabla^2 \ell_s(\theta_{s+1})$,
and $A_t := \gamma H_t + (1-\gamma)\lambda I_d,$
where \(\gamma\in(0,1)\) is the discount factor.
At each round $t$, we update the parameter via:
\begin{equation}
\label{eq:discounted_omd_update}
\theta_{t+1}
\in
\argmin_{\theta\in\Theta}
\left\{
\widetilde{\ell}_t(\theta)
+
\frac{1}{2\eta}\|\theta-\theta_t\|_{A_t}^2
\right\},
\qquad
\text{where}\,\,
\eta := 1+RS,
\end{equation}
and subsequently update the curvature matrix via
\[
H_{t+1}
=
A_t + \nabla^2 \ell_t(\theta_{t+1})
=
A_t + \frac{\mu'(x_t^\top \theta_{t+1})}{g(\tau)}x_t x_t^\top.
\]

\paragraph{Role of the surrogate loss.}
The surrogate loss \(\widetilde{\ell}_t(\theta)\) plays a central role in balancing statistical efficiency and computational efficiency.
Unlike a purely first-order linearization, \(\widetilde{\ell}_t(\theta)\) retains the local curvature of the current loss through \(\nabla^2 \ell_t(\theta_t)\), which is important for capturing the nonlinear geometry of the GLM.
At the same time, it converts the round-\(t\) update into a quadratic optimization problem.
Indeed, after removing terms independent of \(\theta\), \eqref{eq:discounted_omd_update} is equivalent to
\[
\theta_{t+1}
\in
\argmin_{\theta\in\Theta}
\left\{
\left\langle \nabla \ell_t(\theta_t),\,\theta-\theta_t\right\rangle
+
\frac12
\|\theta-\theta_t\|_{\nabla^2 \ell_t(\theta_t)+\eta^{-1}A_t}^2
\right\}.
\]
Thus, the algorithm preserves second-order information from the current sample while avoiding the need to solve a global weighted MLE problem over the entire historical dataset.

\paragraph{Discounted curvature versus stationary curvature.}
To highlight the role of discounting, consider the curvature matrix that would arise in a stationary environment,
$H_t^{\mathrm{stat}}
:=
\lambda I_d+\sum_{s=1}^{t-1}\nabla^2 \ell_s(\theta_{s+1})$ (e.g.,~\citealt{zhang2025generalized}).
In the stationary case, such an undiscounted accumulation is natural, since all past observations are generated under the same underlying parameter and should contribute equally to the local geometry.
In contrast, in a nonstationary environment, old observations are collected under outdated parameters and should therefore have progressively less influence on the current update.
For this reason, we use the discounted curvature matrix $H_t = \lambda I_d+\sum_{s=1}^{t-1}\gamma^{t-1-s}\nabla^2 \ell_s(\theta_{s+1})$,
which exponentially downweights stale curvature information.
This discounting allows the estimator to adapt to drift while still retaining enough historical curvature to stabilize estimation.

\paragraph{Why \(A_t\) instead of \(H_t\)?}
A further distinction, specific to the discounted setting, is that the update rule in~\eqref{eq:discounted_omd_update} uses \(A_t\), not \(H_t\), in the proximal geometry.
This is important.
The matrix \(H_t\) is the \emph{post-round-\((t-1)\)} discounted curvature matrix, whereas the actual geometry needed for the round-\(t\) update is the \emph{pre-update} matrix
$A_t
=
\gamma H_t+(1-\gamma)\lambda I_d
=
\lambda I_d+\sum_{s=1}^{t-1}\gamma^{t-s}\nabla^2 \ell_s(\theta_{s+1}).$
Hence, \(A_t\) is exactly the one-step-aged version of \(H_t\): every past curvature contribution receives one additional factor of \(\gamma\) before the new sample is incorporated.
This is the correct geometry at round \(t\), since historical information from round \(s\) should enter the round-\(t\) update with weight \(\gamma^{t-s}\), not \(\gamma^{t-1-s}\).

\begin{remark}[Constant computation and memory costs]
\label{remark:comp_cost}
The DOMD update has a recursive form: it depends only on the current sample, the current estimate \(\theta_t\), and the maintained curvature matrix \(H_t\).
Thus, it can be implemented without storing or revisiting past observations, requiring only \(\BigO(d^2)\) memory and \(\BigO(d^3)\) computation per round, both independent of \(t\).
By contrast, window-based methods~\citep{russac2020algorithms} must store and process a window of past observations, so their costs scale with the window size \(\tau = \BigO(T^{2/3})\).
Weighted MLE-based methods~\citep{russac2021self,wang2023revisiting} similarly revisit past weighted observations, leading to costs that grow with \(t\).
\end{remark}

\begin{theorem}[Discounted-OMD confidence bound]
\label{thm:DOMD_confidence}
    Let
    $\eta = 1+RS$,
    $\alpha = \frac{3\eta}{2}$ and
    $\iota = \sqrt{\frac{2\eta k_\mu}{g(\tau)}}.$
    Suppose
    $\lambda \ge \max\left\{
    \frac{6\eta R k_\mu S}{g(\tau)},
    \frac{32\alpha dR^2}{7}
    \right\}.$
    For each $t \geq 2$, define the confidence set
    \[
    \Ccal_t(\delta) := \left\{
    \theta \in \RR^d \mid
    \| \theta - \theta_t\|_{H_t}
    \le
    \beta_t(\delta)+\iota \mathcal V_t
    \right\},
    \]
    where
    $\mathcal V_t
    :=
    \sum_{p=1}^{t-1} \gamma^{(t-1-p)/2}\sqrt{\frac{1-\gamma^p}{1-\gamma}} \|\theta_{p+1}^\star-\theta_p^\star\|_2$,
    and
    \[
    \beta_t^2(\delta)
    :=
    4\lambda S^2
    +
    2\eta\left(1+\frac{R^2}{g(\tau)k_\mu}\right)
    \log\!\left(\frac{\pi^2 t^2}{3\delta}\right)
    +
    2\eta\left(3\eta+\frac12\right)
    d\log\!\left(
    1+\frac{k_\mu(1-\gamma^{t-1})}{\lambda d\,g(\tau)(1-\gamma)}
    \right).
    \]
    Then, under Assumptions~\ref{assum:boundedness} and~\ref{assum:bounded_link}, we have $\Pr \left[ \forall t \geq 2,\ \theta_t^\star \in \Ccal_t(\delta) \right] \ge 1-\delta$.
\end{theorem}
\paragraph{Discussion of Theorem~\ref{thm:DOMD_confidence}.}
Theorem~\ref{thm:DOMD_confidence} shows that one can construct an ellipsoidal confidence set \(\Ccal_t(\delta)\) around the DOMD iterate \(\theta_t\) that is both \textit{statistically efficient}, in the sense that the radius has no explicit \(1/c_{\mu}\) dependence, and \textit{computationally efficient}, with constant per-round computation and storage costs.
In contrast to non-discounted OMD~\citep{zhang2025generalized}, which treats all past observations equally and is therefore better suited to stationary environments, our confidence bound explicitly incorporates forgetting through the discounted matrix \(H_t\) and the drift term \(\mathcal V_t\).
As a result, stale data are progressively downweighted, so the confidence set tracks the current parameter more accurately in nonstationary environments.
The proof is deferred to Appendix~\ref{app_sec:proof_confidence}.

\paragraph{Proof sketch of Theorem~\ref{thm:DOMD_confidence}.}
We provide a proof sketch and highlight the technical novelties.

\paragraph{1) Reduction to discounted inverse regret.}
We first upper bound the estimation error of \(\theta_t\) as:
\begin{align}
    \|\theta_t-\theta_t^\star\|_{H_t}^2
    &\le
    4\lambda S^2
    +
    2\eta
    \underbrace{\sum_{s=1}^{t-1}\gamma^{t-1-s}
    \bigl(\ell_s(\theta_t^\star)-\ell_s(\theta_{s+1})\bigr) }_{\text{discounted inverse-regret}}
    -\frac23\sum_{s=1}^{t-1}\gamma^{t-1-s}\|\theta_{s+1}-\theta_s\|_{A_s}^2.
    \label{eq:main_inverse_regret}
\end{align}
Note that the last term on the right-hand side is measured in the \(A_s\)-norm because the DOMD update uses the pre-update geometry \(A_s\), whereas the left-hand side is measured in the post-update matrix \(H_t\).
This geometric mismatch is specific to the discounted setting and is absent in the stationary OMD analysis~\citep{zhang2024online, lee2024nearly, lee2025improved, zhang2025generalized}.

\paragraph{2) Decomposition of discounted inverse regret.}
To analyze the discounted inverse-regret,
we introduce the mix loss
$m_s(\Pi_s)
\!:=\!
-\log \EE_{\theta\sim\Pi_s}\!\left[e^{-\ell_s(\theta)}\right]$,
where
\(\Pi_s \!:=\! \mathcal N(\theta_s,\alpha A_s^{-1})\).
Then, we have
\begin{align*}
    \sum_{s=1}^{t-1}\!\gamma^{t-1-s}
    \bigl(\ell_s(\theta_t^\star)-\ell_s(\theta_{s+1})\bigr)
    &=\underbrace{\sum_{s=1}^{t-1}\!\gamma^{t-1-s}
    \bigl(\ell_s(\theta_s^\star)-m_s(\Pi_s)\bigr)}_{\text{concentration term}}
    \!+\!\sum_{s=1}^{t-1}\!\gamma^{t-1-s}
    \bigl(m_s(\Pi_s)-\ell_s(\theta_{s+1})\bigr)
    \\[0.4ex]
    &+\underbrace{\sum_{s=1}^{t-1}\!\gamma^{t-1-s}
    \bigl(\ell_s(\theta_t^\star)-\ell_s(\theta_s^\star)\bigr)}_{\text{loss-drift term}}.
\end{align*}
Unlike the stationary setting of~\citet{zhang2025generalized}, our decomposition contains an additional \textit{loss-drift} term due to the time-varying comparator \(\theta_t^\star\).
The middle term is controlled by Lemma~\ref{lemma:weighted-mix-gap}, which yields a discounted log-determinant term and a movement term; with \(\alpha=3\eta/2\), the latter is canceled by the negative last term in Equation~\eqref{eq:main_inverse_regret}.
Therefore, the main remaining difficulties are to control the concentration term and the loss-drift term.

\paragraph{3) Fixed-time exponential bound for the concentration term.}
For the concentration term, the stationary OMD analysis of~\citet{zhang2025generalized} based on a supermartingale and Ville's inequality~\citep{ville1939etude} no longer applies.
Indeed, letting
$Z_s:=\ell_s(\theta_s^\star)-m_s(\Pi_s)$
and
$S_t:=\sum_{s=1}^{t-1}\gamma^{t-1-s}Z_s$,
the discounted recursion becomes \(S_{t+1}=\gamma S_t+Z_t\), so the process \(e^{S_t}\) is generally not a supermartingale.
We therefore first derive a \textit{fixed-time} exponential bound, and then obtain a uniform-in-\(t\) statement by applying a union bound.
Specifically, fixing the time \(t\) and writing \(w_{s,t}:=\gamma^{t-1-s}\in(0,1]\), the likelihood-ratio identity gives
\(\EE[e^{Z_s}\mid \widetilde{\mathcal F}_s]=1\),
where \(\widetilde{\mathcal F}_s\)
denotes a suitable filtration; its precise definition is given in the appendix.
Since \(x\mapsto x^{w_{s,t}}\) is concave on \(\mathbb{R}_+\), conditional Jensen yields
$\EE[e^{w_{s,t}Z_s}\mid \widetilde{\mathcal F}_s]
\le
\bigl(\EE[e^{Z_s}\mid \widetilde{\mathcal F}_s]\bigr)^{w_{s,t}} \!=\!1.$
Iterating over \(s\) gives
$\EE[e^{\sum_{s=1}^{t-1}w_{s,t}Z_s}]\le 1$,
and therefore Markov's inequality implies that the concentration term is at most \(\log(1/\delta)\) with probability at least \(1-\delta\).
Thus, unlike the stationary case, concentration is obtained through a fixed-time exponential argument tailored to discounting.
Applying this fixed-time bound with \(\delta_t = 6\delta/(\pi^2 t^2)\), one can then take a union bound over \(t\) to obtain the desired guarantee uniformly over \(t\) (see Lemma~\ref{lemma:weighted-exp}).

\paragraph{4) Control of the loss-drift term.}
The main difficulty is that the drift of the comparator is measured through past losses \(\ell_s\), rather than directly through parameter differences.
This gives rise to an additional loss-space error term that does not appear in the stationary OMD analysis~\citep{zhang2025generalized}.
It is also structurally different from the explicit parameter-bias terms used in nonstationary MLE-based analyses~\citep{wang2023revisiting}.
We control this term directly in loss space by writing
\[
\ell_s(\theta_t^\star)-\ell_s(\theta_s^\star)
=
\underbrace{
\frac{
m(x_s^\top\theta_t^\star)-m(x_s^\top\theta_s^\star)
-\mu(x_s^\top\theta_s^\star)x_s^\top(\theta_t^\star-\theta_s^\star)
}{g(\tau)}}_{\text{deterministic term}}
-
\underbrace{\frac{\eta_s\,x_s^\top(\theta_t^\star-\theta_s^\star)}{g(\tau)}}_{\text{stochastic term}}.
\]
The first term on the right-hand side is deterministic.
By the smoothness of the GLM loss under Assumption~\ref{assum:bounded_link}, it is bounded by
$\frac{k_\mu}{2g(\tau)}
\sum_{s=1}^{t-1}\gamma^{t-1-s}
\bigl(x_s^\top(\theta_t^\star-\theta_s^\star)\bigr)^2
\le
\frac{k_\mu}{2g(\tau)}\mathcal V_t^2.$
Furthermore, the second term is stochastic.
Since the comparator path is \textit{oblivious}, i.e., the entire sequence \(\{\theta_u^\star\}_{u\ge1}\) is independent of the realized rewards, the coefficients
$\gamma^{t-1-s}x_s^\top(\theta_t^\star-\theta_s^\star)$
are predictable.
Therefore, a discounted Hoeffding--Chernoff argument yields a deviation term of order
$\sqrt{\mathcal V_t^2\log(1/\delta)}$.
By Young's inequality, this term is split into a quadratic variation term, which is absorbed into the deterministic variation bound, and a logarithmic confidence penalty.
Combining the deterministic and stochastic parts, we obtain
$\sum_{s=1}^{t-1}\gamma^{t-1-s}
\bigl(\ell_s(\theta_t^\star)-\ell_s(\theta_s^\star)\bigr)
\le
\frac{k_\mu}{g(\tau)}\mathcal V_t^2
+
\frac{R^2}{4g(\tau)k_\mu}\log\frac1\delta.$
Therefore, the loss-drift term is controlled by the discounted variation measure \(\mathcal V_t\) (Lemma~\ref{lemma:drift-loss}).

\subsection{Action Selection}
\label{subsec:action_selection}
We select actions using an implementable optimistic score that does not require prior knowledge of the variation budget \(\mathcal V_t\).
An exact optimism-in-the-face-of-uncertainty (OFU) rule~\citep{abbasi2011improved} over the full confidence set \(\Ccal_t(\delta)\) would include the drift-dependent part of the confidence radius, and hence would not be directly implementable without knowing \(\mathcal V_t\).
We therefore use the following computable UCB-type rule:
\begin{equation}
\label{eq:discounted_ucb_rule}
    x_t
    \in
    \argmax_{x\in \Xcal_t}
    \left\{
        x^\top \theta_t
        +
        \beta_t(\delta)\|x\|_{H_t^{-1}}
    \right\}.
\end{equation}
The full procedure is summarized in Algorithm~\ref{alg:discounted_omd_ucb}. The same algorithm applies to both drifting and piecewise-stationary environments, without requiring prior knowledge of the underlying type of nonstationarity.
\begin{algorithm}[t!]
\caption{\AlgName{}: \textbf{D}iscounted \textbf{O}nline \textbf{M}irror \textbf{D}escent for \textbf{G}eneralized \textbf{L}inear \textbf{B}andit}
\label{alg:discounted_omd_ucb}
\begin{algorithmic}[1]
\Require Discount factor \(\gamma\in(0,1)\), regularization \(\lambda>0\), confidence level \(\delta\).
\State Initialize \(\theta_1\in\Theta\) and \(H_1\gets \lambda I_d\).
\For{\(t=1,2,\dots,T\)}
    \State Compute the pre-update matrix \(A_t \leftarrow \gamma H_t + (1-\gamma)\lambda I_d\).
    \State Select \(x_t \in \Xcal_t\) according to \eqref{eq:discounted_ucb_rule} and observe the reward \(r_t\).
    \State Update the online estimator \(\theta_{t+1}\) according to \eqref{eq:discounted_omd_update}.
    \State Compute the post-update matrix \(H_{t+1} \leftarrow A_t + \nabla^2 \ell_t(\theta_{t+1})\).
\EndFor
\end{algorithmic}
\end{algorithm}
\subsection{Dynamic Regret Bound under Drifting Setting}
\label{subsec:regret_drift}
\begin{theorem}
\label{thm:regret}
    Let
    $T \geq 2$,
    $\lambda
    =
    \max\left\{
    \frac{6\eta R k_\mu S}{g(\tau)},
    \frac{32\alpha dR^2}{7},
    \frac{c_\mu}{g(\tau)}
    \right\}$,
    and
    $\alpha=\frac{3\eta}{2}$.
    Suppose Assumptions~\ref{assum:boundedness} and~\ref{assum:bounded_link} hold.
    Then, for every $\gamma\in[1/T,1)$, with probability at least $1-\delta$,
    \AlgName{} achieves
    \[
    \Regret
    =
    \BigOTilde\!\left(
    \frac{k_\mu}{\sqrt{c_\mu}}\, d\sqrt T
    +
    \frac{k_\mu}{\sqrt{c_\mu}}\, dT\sqrt{1-\gamma}
    +
    \frac{k_\mu^{3/2}}{\sqrt{c_\mu}}\,
    \frac{P_T}{(1-\gamma)^{3/2}}
    \right).
    \]
    \vspace{-0.2cm}
    Moreover, under the tuned choice
    $\gamma
    =
    1-
    \min\left\{
    1-\frac1T,\;
    \max\left\{
    \frac1T,\;
    \sqrt{\frac{\sqrt{k_\mu}\,P_T}{dT}}
    \right\}
    \right\}$,
    we obtain
    \[
    \Regret
    =
    \begin{cases}
    \BigOTilde\!\left(
    \frac{k_\mu}{\sqrt{c_\mu}}\, d\sqrt T
    \right),
    & \text{if } 0\le P_T< \dfrac{d}{\sqrt{k_\mu}\,T},
    \\[1.4ex]
    \BigOTilde\!\left(
    \frac{k_\mu^{9/8}}{\sqrt{c_\mu}}\,
    d^{3/4}P_T^{1/4}T^{3/4}
    \right),
    & \text{if }
    \dfrac{d}{\sqrt{k_\mu}\,T}
    \le
    P_T
    \le
    \dfrac{dT}{\sqrt{k_\mu}}
    \left(1-\dfrac1T\right)^2,
    \\[1.8ex]
    \BigOTilde\!\left(
    \frac{k_\mu^{3/2}}{\sqrt{c_\mu}}\,P_T
    \right),
    & \text{if }
    P_T>
    \dfrac{dT}{\sqrt{k_\mu}}
    \left(1-\dfrac1T\right)^2.
    \end{cases}
    \]
\end{theorem}
\paragraph{Discussion of Theorem~\ref{thm:regret}.}
In the intermediate-\(P_T\) regime, Theorem~\ref{thm:regret} yields the regret bound \(\BigOTilde\big(c_\mu^{-1/2} d^{3/4} P_T^{1/4} T^{3/4}\big)\), matching the best-known regret order of the discounted MLE-based method \texttt{SCB-WeightUCB}~\citep{wang2023revisiting}, while replacing the discounted MLE update with DOMD.
This leads to a substantial computational advantage: our per-round computation and memory costs are constant, both independent of \(t\), whereas weighted MLE-based methods~\citep{faury2021regret, wang2023revisiting} must revisit past weighted observations,
so their costs grow with \(t\).
Moreover, those methods rely on a non-convex projection step, which is generally not guaranteed to be computable in polynomial time.
We also note that \texttt{MASTER}~\citep{wei2021non} can achieve
the sharper regret bound
$\BigOTilde\big(
c_\mu^{-1} d P_T^{1/3} T^{2/3}
\big),$
but for a different problem setting: it assumes a fixed action
set, whereas we allow the feasible action set \(\mathcal X_t\) to vary over
time.
Hence, the two guarantees are not directly comparable, and the existing
\texttt{MASTER} result does not apply to our setting.
In addition, the \texttt{MASTER} meta-algorithm itself incurs
\(\BigO(\log t)\) computation and memory overhead per round, and thus its
computational and storage costs still depend on \(t\), unlike our constant-cost
update.

Among methods applicable, or directly extendable, to time-varying action sets,
our result achieves the best currently supported regret order while providing
substantially improved computational and memory efficiency.
The seemingly sharper \texttt{RestartGLB} bound of
\citet{zhao2020simple} is based on an analysis later shown to contain a
technical flaw; the corrected linear-bandit analysis~\citep{zhao2021non}
instead yields a \(T^{3/4}\)-type, rather than \(T^{2/3}\)-type, regret bound,
while no corrected GLB guarantee has been established.
The proof is deferred to Appendix~\ref{app_sec:proof_thm:regret}.

\subsection{Dynamic Regret Bound under Piecewise-Stationary Setting}
\label{subsec:piecewise_regret}
\begin{theorem}
\label{thm:regret_piecewise}
In the piecewise-stationary setting, under the same parameter choices and assumptions as in Theorem~\ref{thm:regret}, for every \(\gamma\in[1/T,1)\), with probability at least \(1-\delta\), \AlgName{} satisfies
\[
\Regret
=
\BigOTilde\!\left(
\frac{k_\mu}{\sqrt{c_\mu}}\,d\sqrt T
+
\frac{k_\mu}{\sqrt{c_\mu}}\,dT\sqrt{1-\gamma}
+
\frac{\Gamma_T}{1-\gamma}
\right).
\]
\vspace{-0.2cm}
Moreover, for the tuned choice
$\gamma
=
1-
\min\left\{
1-\frac1T,\;
\max\left\{
\frac1T,\;
\left(
\frac{\Gamma_T\sqrt{c_\mu}}{k_\mu dT}
\right)^{2/3}
\right\}
\right\},$
we obtain
\[
\Regret
=
\begin{cases}
\BigOTilde\!\left(
\frac{k_\mu}{\sqrt{c_\mu}}\,d\sqrt T
\right),
&\text{if }0\le \Gamma_T< \dfrac{k_\mu d}{\sqrt{c_\mu T}},
\\[1.4ex]
\BigOTilde\!\left(
\frac{k_\mu^{2/3}}{c_\mu^{1/3}}\,
d^{2/3}\Gamma_T^{1/3}T^{2/3}
\right),
&\text{if }
\dfrac{k_\mu d}{\sqrt{c_\mu T}}
\le
\Gamma_T
\le
\dfrac{k_\mu dT}{\sqrt{c_\mu}}
\left(1-\dfrac1T\right)^{3/2},
\\[1.8ex]
\BigOTilde\!\left(
\Gamma_T
\right),
&\text{if }
\Gamma_T>
\dfrac{k_\mu dT}{\sqrt{c_\mu}}
\left(1-\dfrac1T\right)^{3/2}.
\end{cases}
\]
\end{theorem}
\paragraph{Discussion of Theorem~\ref{thm:regret_piecewise}.}
In the intermediate-\(\Gamma_T\) regime, Theorem~\ref{thm:regret_piecewise} gives the regret bound \(\BigOTilde\big(c_\mu^{-1/3} d^{2/3}\Gamma_T^{1/3}T^{2/3}\big)\).
Compared with the best-known discounted MLE-based regret bound \(\BigOTilde\big(d^{2/3}\Gamma_T^{1/3}T^{2/3}\big)\), achieved by \texttt{SCB-PW-WeightUCB}~\citep{wang2023revisiting}, our result incurs an additional dependence on \(c_\mu\).
However, that method relies on a non-convex projection step and thus does not generally admit a polynomial-time guarantee.
Moreover, since it is MLE-based, its memory cost grows linearly with \(t\).
From this perspective, the additional \(1/c_\mu\) dependence in our bound can be viewed as the cost of achieving substantially better computation and memory efficiency.
We believe that further improving the dependence on \(1/c_\mu\) while retaining an online estimation approach is nontrivial, and we leave this direction for future work.
The proof is deferred to Appendix~\ref{app_sec:proof_thm:regret_piecewise}.

Note that Theorem~\ref{thm:regret} (and Theorem~\ref{thm:regret_piecewise}) explicitly enforce the admissibility constraint on \(\gamma\).
By contrast, prior works~\citep{faury2021regret, wang2023revisiting} present tuned choices of \(\gamma\) without explicit clipping.
These choices can become infeasible when \(P_T\) is large, so their stated tuned bounds effectively apply only in an implicit sublinear-\(P_T\) regime.
If the admissibility of \(\gamma\) were treated explicitly, as in our analysis, the corresponding tuned regret bounds would also take a three-case form (see Remark~\ref{remark:clip_gamma}).

\section{Numerical Experiments}
\begin{figure*}[t!]
    \centering
    \includegraphics[clip, trim=0cm 0.0cm 0cm 0.0cm, width=\textwidth]{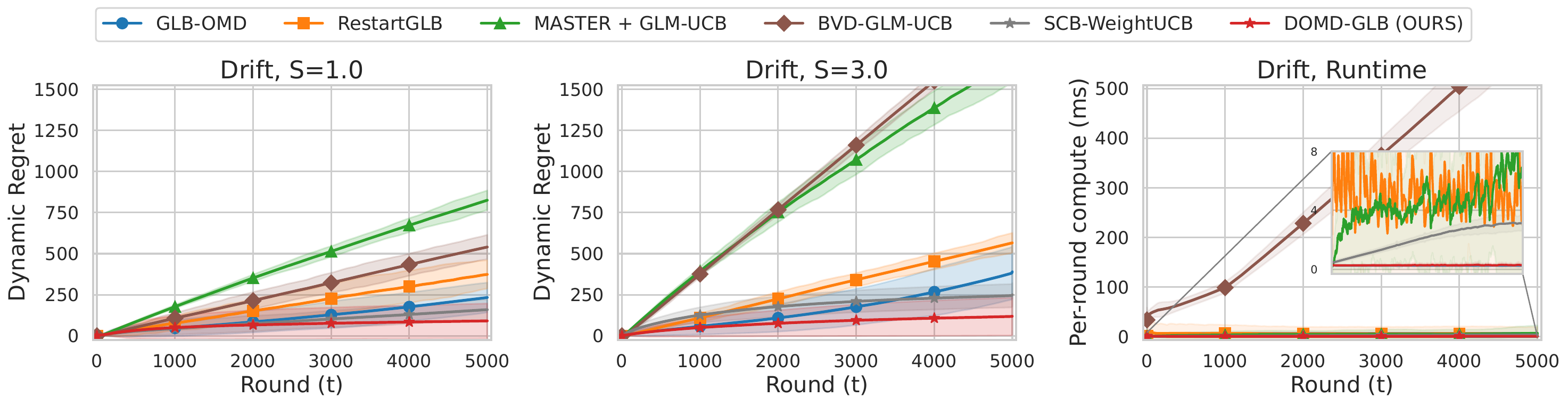}\\[0.5em]
    \includegraphics[clip, trim=0cm 0.0cm 0cm 0.0cm, width=\textwidth]{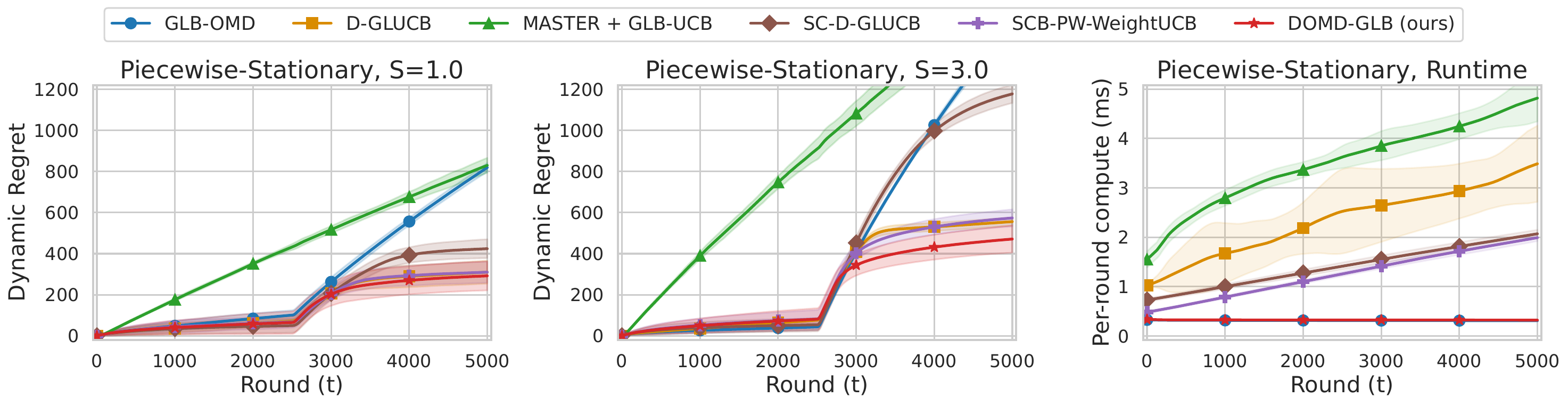}
    \caption{Cumulative dynamic regret and runtime for varying $S$ in two nonstationary settings.}
    \label{fig:experiment_all}
\end{figure*}
In this section, we empirically evaluate the proposed algorithm.
At each round \(t\), the learner is presented with \(N\) arms, whose feature vectors are drawn i.i.d. from a standard Gaussian distribution and normalized to unit \(\ell_2\)-norm.
After selecting an arm \(x_t\), the learner observes a reward \(r_t \sim \mathrm{Bernoulli}(\sigma(x_t^\top \theta_t))\), where \(\sigma(z)=(1+e^{-z})^{-1}\) and \(\theta_t\) is the time-varying parameter.
We consider two types of non-stationarity.
In the \emph{drifting} setting, \(\theta_t\) evolves continuously along a circle of radius \(S\) and completes one full rotation, so the total angular change is \(2\pi\) and the resulting variation budget satisfies \(P_T \le 2S\pi\).
In the \emph{piecewise-stationary} setting, \(\theta_t\) remains constant within each segment and changes once at \(t=T/2\) through a sign flip \(\theta \leftarrow -\theta\), so \(\Gamma_T=1\).
Throughout, we set \(T=5000\), \(d=5\), \(N=30\), and \(S \in \{1.0, 3.0\}\).
We report the mean cumulative dynamic regret and the mean per-round computation time over 20 random seeds, along with one standard deviation for each.

\paragraph{Baselines.}
For the drifting experiments, we compare \AlgName{} with the following baselines: (a) \texttt{GLB-OMD}~\citep{zhang2025generalized}, a stationary OMD-based algorithm; (b) \texttt{RestartGLB}~\citep{zhao2020simple}, a restart-based algorithm; (c) \texttt{MASTER \!\!+\!\! GLM-UCB}~\citep{wei2021non}, an adaptive restart-based algorithm; (d) \texttt{BVD-GLM-UCB}~\citep{faury2021regret}; and (e) \texttt{SCB-WeightUCB}~\citep{wang2023revisiting}.
For the piecewise-stationary experiments, we compare \AlgName{} with: (a) \texttt{GLB-OMD}; (b) \texttt{D-GLUCB}~\citep{russac2020algorithms}; (c) \texttt{MASTER \!\!+\!\! GLM-UCB}; (d) \texttt{SC-D-GLUCB}~\citep{russac2021self}; and (e) \texttt{SCB-PW-WeightUCB}~\citep{wang2023revisiting}.
We use the theoretically prescribed values for $\lambda$, \(\gamma\) and the restart period \(H\).
For all methods, we scale the theoretical confidence radii by \(1/5\), since the original bonuses are too conservative in practice.
For \texttt{MASTER}, we set the block-scale parameter to \(n=12\).
For algorithms that require non-convex projection steps, and therefore do not admit polynomial-time computational guarantees, such as \texttt{BVD-GLM-UCB}, \texttt{SCB-WeightUCB}, \texttt{D-GLUCB}, and \texttt{SCB-PW-WeightUCB}, we use projected gradient descent as an approximation.

\paragraph{Results.}
Figure~\ref{fig:experiment_all} shows that \AlgName{} consistently outperforms
all other baselines across all settings, while incurring only a constant
per-round computational cost.
\texttt{MASTER} performs particularly poorly, with its cumulative regret
remaining nearly linear over the considered horizon.
Importantly, its GLB guarantee is established only for a fixed action set,
whereas our experiments use time-varying feasible action sets.
Thus, \texttt{MASTER} is operated outside the setting covered by its existing
theoretical guarantee, which may partly explain its poor empirical performance.
This behavior is also consistent with the practical limitations of
\texttt{MASTER} reported by~\citet{gerogiannis2025prior}.
Moreover, in the piecewise-stationary setting, although
\texttt{SCB-PW-WeightUCB} achieves a tighter regret bound than ours by a factor
of \(c_{\mu}^{-1/3}\!\), its empirical performance is worse than ours.
This is because its UCB bonus is more conservative, scaling with
\(c_{\mu}^{-1/2}\!\) (Equation~(17) in~\citealt{wang2023revisiting}),
whereas our UCB bonus has no such harmful dependence
(Equation~\eqref{eq:discounted_ucb_rule}).

\section{Conclusion}
\label{sec:conclusion}
In this paper, to the best of our knowledge, we introduce the first algorithm for nonstationary GLBs built on a novel discounted online mirror descent (DOMD) method.
Using this DOMD method, \AlgName{} achieves
$\BigOTilde(
c_\mu^{-1/2}
d^{3/4}P_T^{1/4}T^{3/4}
)$
and
$\BigOTilde(
c_\mu^{-1/3} d^{2/3} \Gamma_T^{1/3}T^{2/3}
)$
dynamic regret bounds in the drifting and piecewise-stationary settings, respectively, while requiring only constant per-round computation and memory costs.
We believe this opens a new and promising research direction toward nonstationary parametric bandit algorithms that are both computationally and statistically efficient.

\bibliography{main_bib}

\clearpage
\appendix
\counterwithin{theorem}{section}
\counterwithin{example}{section}
\counterwithin{equation}{section}
\counterwithin{algorithm}{section}
\counterwithin{figure}{section}
\counterwithin{table}{section}

\part{Appendix}
\parttoc
\clearpage
\allowdisplaybreaks[2]

\section{Proof of Theorem~\ref{thm:DOMD_confidence}}
\label{app_sec:proof_confidence}
In this section, we present the proof of Theorem~\ref{thm:DOMD_confidence}. We first give the main proof (Subsection~\ref{app_subsec:main_proof_thm_confidence}), then provide the proofs of the key lemmas used therein (Subsection~\ref{app_subsec:proof_lemmas_thm_confidence}), and finally present several auxiliary technical lemmas (Subsection~\ref{app_subsec:aux_thm_confidence}).
\subsection{Main Proof of Theorem~\ref{thm:DOMD_confidence}} 
\label{app_subsec:main_proof_thm_confidence}
\begin{proof} [Proof of Theorem~\ref{thm:DOMD_confidence}]
Compared with the stationary analysis of~\citet{zhang2025generalized}, our nonstationary setting requires three essential changes: 
(i) the potential and Gaussian perturbation must be defined using the discounted curvature \(A_t\) rather than the stationary matrix $H_t^{\mathrm{stat}}
:=
\lambda I_d+\sum_{s=1}^{t-1}\nabla^2 \ell_s(\theta_{s+1})$; 
(ii) the discounted exponential process is no longer additive, so Ville's inequality and the resulting anytime concentration argument no longer apply directly (see Remark~\ref{remark:why_fixed}); 
and (iii) the drifting comparator sequence \(\{\theta_t^\star\}_{t \geq 1}\) creates an additional drift term that must be controlled separately.

We begin the proof by stating the one-step discounted inequality lemma.
\begin{lemma}[One-step discounted OMD inequality]
    \label{lemma:onestep}
    Let $\eta = 1 + RS$
    and $\lambda\ge
    \frac{6\eta R k_\mu S}{g(\tau)}$.
    Then, under Assumptions~\ref{assum:boundedness} and~\ref{assum:bounded_link},
    for every $\theta\in\Theta$ and every $t\ge 1$, we have
    \[
        \|\theta_{t+1}-\theta\|_{H_{t+1}}^2
        \le
        \|\theta_t-\theta\|_{A_t}^2
        +
        2\eta\bigl(\ell_t(\theta)-\ell_t(\theta_{t+1})\bigr)
        -
        \frac23\|\theta_{t+1}-\theta_t\|_{A_t}^2.
    \]
    Consequently, for every $\theta\in\Theta$ and every $t\ge 2$,
    \begin{equation*}
    \|\theta_t-\theta\|_{H_t}^2
    \le
    4\lambda S^2
    +
    2\eta\sum_{s=1}^{t-1}\gamma^{t-1-s}\bigl(\ell_s(\theta)-\ell_s(\theta_{s+1})\bigr)
    -
    \frac23\sum_{s=1}^{t-1}\gamma^{t-1-s}\|\theta_{s+1}-\theta_s\|_{A_s}^2.
    \end{equation*}
\end{lemma}
The proof is provided in Appendix~\ref{app_subsubsec:proof_lemma:onestep}.

By applying Lemma~\ref{lemma:onestep} with $\theta=\theta_t^\star$,
we get
\begin{align}
    \|\theta_t-\theta_t^\star\|_{H_t}^2
    \le
    4\lambda S^2
    +
    2\eta\sum_{s=1}^{t-1}\gamma^{t-1-s}
    \bigl(\ell_s(\theta_t^\star)-\ell_s(\theta_{s+1})\bigr)
    -
    \frac23\sum_{s=1}^{t-1}\gamma^{t-1-s}\|\theta_{s+1}-\theta_s\|_{A_s}^2.
    \label{eq:onestep}
\end{align}
For each round \(t\), let
$\Pi_t := \mathcal N(\theta_t,\alpha A_t^{-1})$,
where $\alpha=3\eta/2$, 
denote the Gaussian distribution centered at the current iterate \(\theta_t\), with covariance matrix \(\alpha A_t^{-1}\).
Associated with this Gaussian perturbation, we define the mix loss as follows:
\[
m_t(\Pi_t)
:=
-\log \EE_{\theta\sim \Pi_t}\!\left[e^{-\ell_t(\theta)}\right].
\]
Now, we decompose $\sum_{s=1}^{t-1}\gamma^{t-1-s}
    \bigl(\ell_s(\theta_t^\star)-\ell_s(\theta_{s+1})\bigr)$ in~\eqref{eq:onestep} as follows:
\begin{align*}
    &\sum_{s=1}^{t-1}\!\gamma^{t-1-s}
    \bigl(\ell_s(\theta_t^\star)-\ell_s(\theta_{s+1})\bigr)
    \\
    &\!=\sum_{s=1}^{t-1}\!\gamma^{t-1-s}
    \bigl(\ell_s(\theta_s^\star)-m_s(\Pi_s)\bigr)
    \!+\!\sum_{s=1}^{t-1}\!\gamma^{t-1-s}
    \bigl(m_s(\Pi_s)-\ell_s(\theta_{s+1})\bigr)
    \\[0.4ex]
    &\quad+\sum_{s=1}^{t-1}\!\gamma^{t-1-s}
    \bigl(\ell_s(\theta_t^\star)-\ell_s(\theta_s^\star)\bigr).
    \numberthis
    \label{eq:onestep_decomposition}
\end{align*}
We next present three lemmas that bound the above three terms.
%
\begin{lemma}[Fixed-time discounted mix-loss concentration]
    \label{lemma:weighted-exp}
    Define $\Pi_t := \mathcal N(\theta_t,\alpha A_t^{-1})$ and
    $m_t(\Pi_t)
    :=
    -\log \EE_{\theta\sim \Pi_t}\!\left[e^{-\ell_t(\theta)}\right]$.
    Then, for any fixed $t\ge 2$, with probability at least $1-\delta$, we have
    \[
        \sum_{s=1}^{t-1}\gamma^{t-1-s}\bigl(\ell_s(\theta_s^\star)-m_s(\Pi_s)\bigr)
        \le
        \log\frac{1}{\delta}.
    \]
\end{lemma}
The proof is provided in Appendix~\ref{app_subsubsec:proof_lemma:weighted-exp}.
\begin{remark}[Why the discounted case only yields a fixed-time bound]
\label{remark:why_fixed}
Our exponential argument differs from Lemma 5 of~\citet{zhang2025generalized}. 
In the undiscounted setting of~\citet{zhang2025generalized}, the cumulative process is additive in time, which allows one to construct a nonnegative supermartingale and then apply Ville's inequality~\citep{ville1939etude} to obtain an anytime bound. 
By contrast, in our discounted setting, for a fixed terminal time $t$ we consider $\sum_{s=1}^{t-1}\gamma^{t-1-s} Z_s$, where $Z_s := \ell_s(\theta_s^\star)-m_s(\Pi_s)$. 
For each fixed $t$, one can show that $\EE[e^{\sum_{s=1}^{t-1}\gamma^{t-1-s} Z_s}] \le 1$, and hence a fixed-time high-probability bound follows by Markov's inequality. 
However, the process $(e^{\sum_{s=1}^{t-1}\gamma^{t-1-s} Z_s})_{t\ge 2}$ is generally not a supermartingale, so Ville's inequality cannot be applied.
To see this, define $S_t:= \sum_{s=1}^{t-1}\gamma^{t-1-s} Z_s$.
Then, the obstruction is the discounted recursion $S_{t+1} = \gamma S_t + Z_t$, which gives $e^{S_{t+1}} = e^{\gamma S_t} e^{Z_t}$. 
Since $\EE[e^{Z_t}| \widetilde{\mathcal F}_t]=1$, we get $\EE[e^{S_{t+1}}| \widetilde{\mathcal F}_t] = e^{\gamma S_t}$. 
A supermartingale argument would require $\EE[e^{S_{t+1}}| \widetilde{\mathcal F}_t] \le e^{S_t}$, but this is not guaranteed because $\gamma<1$ and $S_t$ need not be nonnegative.
Therefore, unlike Lemma 5 of~\citet{zhang2025generalized}, our discounted analysis yields concentration only for each fixed $t$, rather than an anytime bound uniform over all $t$.
\end{remark}
\begin{lemma}[Discounted mix-loss gap bound]
    \label{lemma:weighted-mix-gap}
    Define $\Pi_t := \mathcal N(\theta_t,\alpha A_t^{-1})$ and
    $m_t(\Pi_t)
    :=
    -\log \EE_{\theta\sim \Pi_t}\!\left[e^{-\ell_t(\theta)}\right]$.
    Let $\lambda \geq 32 \alpha d R^2 /7$.
    Then, for every $t\ge 2$,
    under Assumption~\ref{assum:boundedness}, 
    we have
    \[
    \begin{aligned}
    \sum_{s=1}^{t-1}\gamma^{t-1-s}\bigl(m_s(\Pi_s)-\ell_s(\theta_{s+1})\bigr)
    &\le
    \left(3\eta+\frac12\right)
    d\log\!\left(
    1+\frac{k_\mu(1-\gamma^{t-1})}{\lambda d\,g(\tau)(1-\gamma)}
    \right)
    \\
    &\quad+
    \frac{1}{3\eta}\sum_{s=1}^{t-1}\gamma^{t-1-s}\|\theta_{s+1}-\theta_s\|_{A_s}^2 
    .
    \end{aligned}
    \]
\end{lemma}
The proof is provided in Appendix~\ref{app_subsubsec:proof_lemma:weighted-mix-gap}.
\begin{lemma}[Drift-loss bound]
\label{lemma:drift-loss}
    Under Assumptions~\ref{assum:boundedness} and~\ref{assum:bounded_link}, for \(t\ge 2\), with probability at least \(1-\delta\), we have
    \[
    \sum_{s=1}^{t-1}\gamma^{t-1-s}
    \bigl(\ell_s(\theta_t^\star)-\ell_s(\theta_s^\star)\bigr)
    \le
    \frac{k_\mu}{g(\tau)}\,\mathcal V_t^2
    +
    \frac{R^2}{4\,g(\tau)\,k_\mu}\log\frac{1}{\delta},
    \]
    where $\mathcal V_t
    =
    \sum_{p=1}^{t-1}
    \gamma^{(t-1-p)/2}
    \sqrt{\frac{1-\gamma^p}{1-\gamma}} \| \theta^\star_{p+1} - \theta^\star_p \|_2$.
\end{lemma}
The proof is provided in Appendix~\ref{app_subsubsec:proof_lemma:drift-loss}.

Consequently, on the right-hand side of~\eqref{eq:onestep_decomposition},
the first term is bounded by Lemma~\ref{lemma:weighted-exp} with probability $1-\delta/2$;
the second term is bounded deterministically by Lemma~\ref{lemma:weighted-mix-gap};
the third term is bounded by Lemma~\ref{lemma:drift-loss} with probability $1-\delta/2$.
Substituting these three bounds into \eqref{eq:onestep} and using
$2\eta\cdot \frac{1}{2\alpha}
=
2\eta\cdot \frac{1}{3\eta}
=
\frac23,$
we obtain, on an event of probability at least $1-\delta$,
\begin{align*}
    \|\theta_t-\theta_t^\star\|_{H_t}^2
    &\le
    4\lambda S^2
    +
    2\eta\log\frac{2}{\delta}
    + 2\eta\left(3\eta+\frac12\right)
    d\log\!\left(
    1+\frac{k_\mu(1-\gamma^{t-1})}{\lambda d\,g(\tau)(1-\gamma)}
    \right)
     \\
    &\quad
    +
    \frac{2\eta R^2}{g(\tau)k_\mu}\log\frac{2}{\delta}
    +
    \frac{2\eta k_\mu}{g(\tau)}\mathcal V_t^2,
    \qquad \forall t \geq 2
    .
\end{align*}
To obtain an anytime statement, for each \(t\ge 2\), set
$\delta_t:=\frac{6\delta}{\pi^2 t^2}.$
Applying the above fixed-time bound with \(\delta\) replaced by \(\delta_t\), we obtain that, for each fixed \(t\ge 2\), the desired inequality holds with probability at least \(1-\delta_t\). Moreover,
\[
\sum_{t=2}^\infty \delta_t
=
\frac{6\delta}{\pi^2}\sum_{t=2}^\infty \frac{1}{t^2}
\le
\frac{6\delta}{\pi^2}\sum_{t=1}^\infty \frac{1}{t^2}
=
\delta.
\]
Therefore, by a union bound over all \(t\ge 2\), with probability at least \(1-\delta\), the stated inequality holds simultaneously for all \(t\ge 2\). This completes the proof of Theorem~\ref{thm:DOMD_confidence}.
\end{proof}

\subsection{Proofs of Lemmas for Theorem~\ref{thm:DOMD_confidence}} 
\label{app_subsec:proof_lemmas_thm_confidence}
\subsubsection{Proof of Lemma~\ref{lemma:onestep}}
\label{app_subsubsec:proof_lemma:onestep}
\begin{proof}[Proof of Lemma~\ref{lemma:onestep}]
    Fix \(t\ge 1\) and \(\theta\in\Theta\). Define
    \[
    f_t(z):=\frac{m(z)-r_t z}{g(\tau)},
    \qquad z\in\RR.
    \]
    Then, for every \(\theta\in\Theta\),
    \[
    \ell_t(\theta)=f_t(x_t^\top\theta),
    \qquad
    f_t'(z)=\frac{\mu(z)-r_t}{g(\tau)},
    \qquad
    f_t''(z)=\frac{\mu'(z)}{g(\tau)},
    \qquad
    f_t'''(z)=\frac{\mu''(z)}{g(\tau)}.
    \]
    Hence, by Assumption~\ref{assum:bounded_link}, we have
    \[
    \frac{c_\mu}{g(\tau)}\le f_t''(z)\le \frac{k_\mu}{g(\tau)},
    \qquad
    |f_t'''(z)|\le \frac{Rk_\mu}{g(\tau)}.
    \]
    
    We first lower bound \(\ell_t(\theta)-\ell_t(\theta_{t+1})\).
    By the integral form of the Bregman divergence,
    \begin{align*}
    \ell_t(\theta)-\ell_t(\theta_{t+1})
    &=
    f_t(x_t^\top\theta)-f_t(x_t^\top\theta_{t+1}) \\
    &=
    f_t'(x_t^\top\theta_{t+1})\,x_t^\top(\theta-\theta_{t+1})
    \\
    &\quad
    +
    \bigl(x_t^\top(\theta-\theta_{t+1})\bigr)^2
    \int_0^1 (1-s)\,
    f_t''\!\left(x_t^\top\theta_{t+1}+s\,x_t^\top(\theta-\theta_{t+1})\right)\,\dd s.
    \end{align*}
    Using \(f_t''=\mu'/g(\tau)\) and applying Lemma~\ref{lemma:self-concor-control} with
    \[
    v=x_t^\top\theta_{t+1},
    \qquad
    z_1=x_t^\top\theta,
    \qquad
    z_2=x_t^\top\theta_{t+1},
    \]
    we obtain
    \[
    \int_0^1 (1-s)\,
    f_t''\!\left(x_t^\top\theta_{t+1}+s\,x_t^\top(\theta-\theta_{t+1})\right)\,\dd s
    \ge
    \frac{\mu'(x_t^\top\theta_{t+1})}
    {g(\tau)\bigl(2+R|x_t^\top(\theta-\theta_{t+1})|\bigr)}.
    \]
    Therefore,
    \begin{align*}
    \ell_t(\theta)-\ell_t(\theta_{t+1})
    &\ge
    f_t'(x_t^\top\theta_{t+1})\,x_t^\top(\theta-\theta_{t+1})
    +
    \frac{\mu'(x_t^\top\theta_{t+1})}
    {g(\tau)\bigl(2+R|x_t^\top(\theta-\theta_{t+1})|\bigr)}
    \bigl(x_t^\top(\theta-\theta_{t+1})\bigr)^2.
    \end{align*}
    Since \(\theta,\theta_{t+1}\in\Theta\) and \(\|x_t\|_2\le 1\),
    \[
    |x_t^\top(\theta-\theta_{t+1})|
    \le
    \|\theta-\theta_{t+1}\|_2
    \le
    2S.
    \]
    Hence, by the definition of $\eta = 1 + RS$, we get
    \begin{align*}
    \ell_t(\theta)-\ell_t(\theta_{t+1})
    &\ge
    f_t'(x_t^\top\theta_{t+1})\,x_t^\top(\theta-\theta_{t+1})
    +
    \frac{1}{2\eta}\frac{\mu'(x_t^\top\theta_{t+1})}{g(\tau)}
    \bigl(x_t^\top(\theta-\theta_{t+1})\bigr)^2
    \\
    &=
    f_t'(x_t^\top\theta_{t+1})\,x_t^\top(\theta-\theta_{t+1})
    +
    \frac{1}{2\eta}
    \|\theta-\theta_{t+1}\|_{\nabla^2\ell_t(\theta_{t+1})}^2.
    \numberthis
    \label{eq:one_step_loss_lower}
    \end{align*}
    
    Next, expand \(f_t'(x_t^\top\theta_{t+1})\) around \(x_t^\top\theta_t\).
    By Taylor's theorem, there exists \(\xi_t\) between \(x_t^\top\theta_t\) and \(x_t^\top\theta_{t+1}\) such that
    \[
    f_t'(x_t^\top\theta_{t+1})
    =
    f_t'(x_t^\top\theta_t)
    +
    f_t''(x_t^\top\theta_t)\,x_t^\top(\theta_{t+1}-\theta_t)
    +
    \frac12 f_t'''(\xi_t)\bigl(x_t^\top(\theta_{t+1}-\theta_t)\bigr)^2.
    \]
    Therefore,
    \begin{align*}
    &f_t'(x_t^\top\theta_{t+1})\,x_t^\top(\theta-\theta_{t+1})
    \\
    &\ge
    \Bigl(
    f_t'(x_t^\top\theta_t)
    +
    f_t''(x_t^\top\theta_t)\,x_t^\top(\theta_{t+1}-\theta_t)
    \Bigr)
    x_t^\top(\theta-\theta_{t+1})
    \\
    &\quad-
    \frac{Rk_\mu}{2g(\tau)}
    \bigl(x_t^\top(\theta_{t+1}-\theta_t)\bigr)^2
    |x_t^\top(\theta-\theta_{t+1})|
    \\
    &\ge
    \Bigl(
    f_t'(x_t^\top\theta_t)
    +
    f_t''(x_t^\top\theta_t)\,x_t^\top(\theta_{t+1}-\theta_t)
    \Bigr)
    x_t^\top(\theta-\theta_{t+1})
    \\
    &\quad-
    \frac{Rk_\mu S}{g(\tau)}
    \bigl(x_t^\top(\theta_{t+1}-\theta_t)\bigr)^2,
    \numberthis 
    \label{eq:one_step_taylor}
    \end{align*}
    where the last inequality holds because $|x_t^\top(\theta-\theta_{t+1})|\le 2S$ (Assumption~\ref{assum:boundedness}).
    
    Moreover, by the definitions of $f_t$ and $\ell_t$, we have
    \begin{align*}
    f_t'(x_t^\top\theta_t)\,x_t^\top(\theta-\theta_{t+1})
    =
    \langle \nabla\ell_t(\theta_t),\,\theta-\theta_{t+1}\rangle,
    \numberthis
    \label{eq:one_step_f'}
    \end{align*}
    and
    \begin{align*}
    f_t''(x_t^\top\theta_t)\,x_t^\top(\theta_{t+1}-\theta_t)\,x_t^\top(\theta-\theta_{t+1})
    =
    \Bigl\langle
    \nabla^2\ell_t(\theta_t)(\theta_{t+1}-\theta_t),\,
    \theta-\theta_{t+1}
    \Bigr\rangle.
    \numberthis
    \label{eq:one_step_loss_f''}
    \end{align*}
    By combining \eqref{eq:one_step_loss_lower}, \eqref{eq:one_step_taylor}, \eqref{eq:one_step_f'}, and \eqref{eq:one_step_loss_f''}, we obtain
    \begin{align*}
    \ell_t(\theta)-\ell_t(\theta_{t+1})
    &\ge
    \Bigl\langle
    \nabla\ell_t(\theta_t)
    +
    \nabla^2\ell_t(\theta_t)(\theta_{t+1}-\theta_t),\,
    \theta-\theta_{t+1}
    \Bigr\rangle
    \\
    &\quad
    +
    \frac{1}{2\eta}
    \|\theta-\theta_{t+1}\|_{\nabla^2\ell_t(\theta_{t+1})}^2
    -
    \frac{Rk_\mu S}{g(\tau)}
    \bigl(x_t^\top(\theta_{t+1}-\theta_t)\bigr)^2
    \\
    &
    \ge
    \langle \nabla\widetilde\ell_t(\theta_{t+1}),\,\theta-\theta_{t+1}\rangle
    +
    \frac{1}{2\eta}
    \|\theta-\theta_{t+1}\|_{\nabla^2\ell_t(\theta_{t+1})}^2
    -
    \frac{Rk_\mu S}{g(\tau)}
    \bigl(x_t^\top(\theta_{t+1}-\theta_t)\bigr)^2.
    \tag{$\nabla\widetilde\ell_t(\theta_{t+1})
    =
    \nabla\ell_t(\theta_t)
    +
    \nabla^2\ell_t(\theta_t)(\theta_{t+1}-\theta_t)$}
    \end{align*}
    Multiplying both sides by \(2\eta\), we get
    \begin{align*}
    2\eta\bigl(\ell_t(\theta)-\ell_t(\theta_{t+1})\bigr)
    &\ge
    2\eta \,\langle \nabla\widetilde\ell_t(\theta_{t+1}),\,\theta-\theta_{t+1}\rangle
    +
    \|\theta-\theta_{t+1}\|_{\nabla^2\ell_t(\theta_{t+1})}^2
    \\
    &\quad
    -
    \frac{2\eta Rk_\mu S}{g(\tau)}
    \bigl(x_t^\top(\theta_{t+1}-\theta_t)\bigr)^2.
    \end{align*}
    Because \(\|x_t\|_2\le 1\), \(A_t\succeq \lambda I_d\), and 
    $\lambda \ge \frac{6\eta Rk_\mu S}{g(\tau)}$,
    we have
    \[
    \frac{2\eta Rk_\mu S}{g(\tau)}
    \bigl(x_t^\top(\theta_{t+1}-\theta_t)\bigr)^2
    \le
    \frac{2\eta Rk_\mu S}{g(\tau)} \|\theta_{t+1}-\theta_t\|_2^2
    \le
    \frac{2\eta Rk_\mu S}{g(\tau)} \frac{1}{\lambda}\|\theta_{t+1}-\theta_t\|_{A_t}^2
    \le
    \frac13\|\theta_{t+1}-\theta_t\|_{A_t}^2.
    \]
    Hence, we obtain
    \begin{align*}
    2\eta\bigl(\ell_t(\theta)-\ell_t(\theta_{t+1})\bigr)
    &\ge
    2\eta \, \langle \nabla\widetilde\ell_t(\theta_{t+1}),\,\theta-\theta_{t+1}\rangle
    +
    \|\theta-\theta_{t+1}\|_{\nabla^2\ell_t(\theta_{t+1})}^2
    -
    \frac13\|\theta_{t+1}-\theta_t\|_{A_t}^2.
    \numberthis
    \label{eq:one_step_loss_gap_lower}
    \end{align*}
    Now, for any $ \theta\in\Theta$, we  define
    $F_t(\theta):=
    \widetilde\ell_t(\theta)
    +
    \frac{1}{2\eta}\|\theta-\theta_t\|_{A_t}^2$.
    Since \(F_t\) is convex and differentiable on the convex set \(\Theta\), and
    \(\theta_{t+1}\in\argmin_{\theta\in\Theta}F_t(\theta)\), we have
    \[
    \bigl\langle \nabla F_t(\theta_{t+1}),\,\theta-\theta_{t+1}\bigr\rangle\ge 0.
    \]
    That is,
    \begin{align*}
    \eta \,\langle \nabla\widetilde\ell_t(\theta_{t+1}),\,\theta-\theta_{t+1}\rangle
    \ge
    -\bigl\langle A_t(\theta_{t+1}-\theta_t),\,\theta-\theta_{t+1}\bigr\rangle.
    \numberthis
    \label{eq:one_step_optimality}
    \end{align*}
    Substituting \eqref{eq:one_step_optimality} into \eqref{eq:one_step_loss_gap_lower} yields
    \begin{align*}
    2\eta\bigl(\ell_t(\theta)-\ell_t(\theta_{t+1})\bigr)
    &\ge
    -2\,\bigl\langle A_t(\theta_{t+1}-\theta_t),\,\theta-\theta_{t+1}\bigr\rangle
    +
    \|\theta-\theta_{t+1}\|_{\nabla^2\ell_t(\theta_{t+1})}^2
    -
    \frac13\|\theta_{t+1}-\theta_t\|_{A_t}^2.
    \numberthis
    \label{eq:one_step_loss_gap_lower2}
    \end{align*}
    
    Finally, since
    $\theta-\theta_t
    =
    (\theta-\theta_{t+1})+(\theta_{t+1}-\theta_t)$,
    the three-point identity gives
    \[
    -2\,\bigl\langle A_t(\theta_{t+1}-\theta_t),\,\theta-\theta_{t+1}\bigr\rangle
    =
    \|\theta-\theta_{t+1}\|_{A_t}^2
    -
    \|\theta-\theta_t\|_{A_t}^2
    +
    \|\theta_{t+1}-\theta_t\|_{A_t}^2.
    \]
    Plugging this into \eqref{eq:one_step_loss_gap_lower2}, we have
    \begin{align*}
    2\eta\bigl(\ell_t(\theta)-\ell_t(\theta_{t+1})\bigr)
    &\ge
    \|\theta-\theta_{t+1}\|_{A_t}^2
    -
    \|\theta-\theta_t\|_{A_t}^2
    +
    \frac23\|\theta_{t+1}-\theta_t\|_{A_t}^2
    +
    \|\theta-\theta_{t+1}\|_{\nabla^2\ell_t(\theta_{t+1})}^2.
    \end{align*}
    Since
    $H_{t+1}=A_t+\nabla^2\ell_t(\theta_{t+1})$,
    we conclude that
    \begin{align*}
    \|\theta_{t+1}-\theta\|_{H_{t+1}}^2
    \le
    \|\theta_t-\theta\|_{A_t}^2
    +
    2\eta\bigl(\ell_t(\theta)-\ell_t(\theta_{t+1})\bigr)
    -
    \frac23\|\theta_{t+1}-\theta_t\|_{A_t}^2.
    \end{align*}
    This proves the one-step inequality, i.e., the first statement.
    
    To unroll this bound, using
    $A_t=\gamma H_t+(1-\gamma)\lambda I_d$,
    we get
    \[
    \|\theta_t-\theta\|_{A_t}^2
    \le
    \gamma\|\theta_t-\theta\|_{H_t}^2
    +
    (1-\gamma)\lambda\|\theta_t-\theta\|_2^2
    \le
    \gamma\|\theta_t-\theta\|_{H_t}^2+4(1-\gamma)\lambda S^2,
    \]
    where the last step uses \(\theta_t,\theta\in\Theta\).
    Therefore,
    \[
    \|\theta_{t+1}-\theta\|_{H_{t+1}}^2
    \le
    \gamma\|\theta_t-\theta\|_{H_t}^2
    +
    4(1-\gamma)\lambda S^2
    +
    2\eta\bigl(\ell_t(\theta)-\ell_t(\theta_{t+1})\bigr)
    -
    \frac23\|\theta_{t+1}-\theta_t\|_{A_t}^2.
    \]
    Iterating this inequality from \(s=1\) to \(t-1\), we obtain
    \begin{align*}
    \|\theta_t-\theta\|_{H_t}^2
    &\le
    \gamma^{t-1}\|\theta_1-\theta\|_{H_1}^2
    +
    4(1-\gamma)\lambda S^2\sum_{s=1}^{t-1}\gamma^{t-1-s}
    \\
    &\quad
    +
    2\eta\sum_{s=1}^{t-1}\gamma^{t-1-s}
    \bigl(\ell_s(\theta)-\ell_s(\theta_{s+1})\bigr)
    -
    \frac23\sum_{s=1}^{t-1}\gamma^{t-1-s}
    \|\theta_{s+1}-\theta_s\|_{A_s}^2.
    \numberthis 
    \label{eq:one_step_loss_gap_iterate}
    \end{align*}
    Since \(H_1=\lambda I_d\) and \(\theta_1,\theta\in\Theta\),
    \[
    \|\theta_1-\theta\|_{H_1}^2
    =
    \lambda\|\theta_1-\theta\|_2^2
    \le
    4\lambda S^2.
    \]
    Therefore,
    \[
    \gamma^{t-1}\|\theta_1-\theta\|_{H_1}^2
    +
    4(1-\gamma)\lambda S^2\sum_{s=1}^{t-1}\gamma^{t-1-s}
    \le
    4\gamma^{t-1} \lambda S^2
    + 4 (1-\gamma) \lambda S^2 \cdot \frac{1- \gamma^{t-1}}{1-\gamma}
    \le
    4\lambda S^2.
    \]
    Substituting this into~\eqref{eq:one_step_loss_gap_iterate} gives
    \[
    \|\theta_t-\theta\|_{H_t}^2
    \le
    4\lambda S^2
    +
    2\eta\sum_{s=1}^{t-1}\gamma^{t-1-s}
    \bigl(\ell_s(\theta)-\ell_s(\theta_{s+1})\bigr)
    -
    \frac23\sum_{s=1}^{t-1}\gamma^{t-1-s}
    \|\theta_{s+1}-\theta_s\|_{A_s}^2.
    \]
    This proves the second statement.
\end{proof}
\subsubsection{Proof of Lemma~\ref{lemma:weighted-exp}}
\label{app_subsubsec:proof_lemma:weighted-exp}
\begin{proof}[Proof of Lemma~\ref{lemma:weighted-exp}]
    Fix $t\ge 2$.
    For simplicity,  let
    $w_{s,t}:=\gamma^{t-1-s}\in(0,1]$ and
    $Z_s:=\ell_s(\theta_s^\star)-m_s(\Pi_s)$.
    Recall that
    $\ell_s(\theta)
    =
    \frac{m(x_s^\top \theta)-r_s x_s^\top \theta}{g(\tau)}
    =
    -\log \operatorname{Pr} \left[ r_s\mid x_s, \theta \right]+h(r_s,\tau)$.
    Hence, 
    we have
    \[
    m_s(\Pi_s)
    =
    -\log \int e^{-\ell_s(\theta)}\, \dd\Pi_s(\theta)
    =
    h(r_s,\tau)-\log\!\int p(r_s\mid x_s^\top\theta)\,\dd\Pi_s(\theta).
    \]
    Therefore, we can write $Z_s$ as:
    \[
    Z_s
    =
    \log\frac{\int p(r_s\mid x_s^\top\theta)\,\dd\Pi_s(\theta)}
    {p(r_s\mid x_s^\top\theta_s^\star)}.
    \]
    Define the filtration before observing $r_t$ by
    $\mathcal F_t
    :=
    \sigma(x_1,r_1,\dots,x_{t-1},r_{t-1},x_t),$
    and the enlarged filtration
    $\widetilde{\mathcal F}_t
    :=
    \mathcal F_t \vee \sigma(\theta_s^\star:s\ge 1)$.
    Then, 
    since \(\Pi_s\) is \(\widetilde{\mathcal F}_s\)-measurable and
    \(r_s\mid \widetilde{\mathcal F}_s \sim p(\cdot\mid x_s^\top\theta_s^\star)\),
    we obtain
    \begin{align*}
    \EE[e^{Z_s}\mid \widetilde{\mathcal F}_s]
    &=
    \EE\!\left[
    \frac{\int p(r_s\mid x_s^\top\theta)\,\dd\Pi_s(\theta)}
    {p(r_s\mid x_s^\top\theta_s^\star)}
    \;\middle|\;
    \widetilde{\mathcal F}_s
    \right] \\
    &=
    \int
    \frac{\int p(r\mid x_s^\top\theta)\,\dd\Pi_s(\theta)}
    {p(r\mid x_s^\top\theta_s^\star)}
    \,p(r\mid x_s^\top\theta_s^\star)\,\dd r \\
    &=
    \int \int p(r\mid x_s^\top\theta)\,\dd\Pi_s(\theta)\,\dd r \\
    &=
    \int 1\,\dd\Pi_s(\theta)=1.
    \end{align*}
    Because $0<w_{s,t}\le 1$, the map $x\mapsto x^{w_{s,t}}$ is concave on $\RR_+$.
    Therefore, conditional Jensen implies
    \begin{align*}
    \EE[e^{w_{s,t}Z_s}\mid \widetilde{\mathcal F}_s]
    =
    \EE[(e^{Z_s})^{w_{s,t}}\mid \widetilde{\mathcal F}_s]
    \le
    \bigl(\EE[e^{Z_s}\mid \widetilde{\mathcal F}_s]\bigr)^{w_{s,t}}
    =
    1.
    \numberthis 
    \label{eq:lemma:weighted-exp_jensen}
    \end{align*}
    
    We now iterate conditional expectations.
    By \eqref{eq:lemma:weighted-exp_jensen}, we get
    \begin{align*}
    \EE\!\left[
    e^{\sum_{s=1}^{t-2}w_{s,t}Z_s}\,
    \EE[e^{w_{t-1,t}Z_{t-1}}\mid \widetilde{\mathcal F}_{t-1}]
    \right] 
    \le
    \EE\!\left[e^{\sum_{s=1}^{t-2}w_{s,t}Z_s}\right].
    \end{align*}
    Applying the same argument recursively to $t-2,t-3,\dots,1$ and using \eqref{eq:lemma:weighted-exp_jensen} at each step, we get
    \[
    \EE[e^{\sum_{s=1}^{t-1}w_{s,t}Z_s}] \le 1.
    \]
    Hence, by Markov's inequality, we obtain
    \[
    \operatorname{Pr}\!\left(\sum_{s=1}^{t-1}w_{s,t}Z_s\ge \log\frac1\delta\right)
    =
    \operatorname{Pr}\!\left(e^{\sum_{s=1}^{t-1}w_{s,t}Z_s}\ge \frac1\delta\right)
    \le
    \delta \EE[e^{\sum_{s=1}^{t-1}w_{s,t}Z_s}]
    \le \delta.
    \]
    Therefore, with probability at least $1-\delta$, we have
    \[
    \sum_{s=1}^{t-1}w_{s,t}Z_s
    \le
    \log\frac1\delta.
    \]
    Recalling $w_{s,t}=\gamma^{t-1-s}$ 
    and
    $Z_s=\ell_s(\theta_s^\star)-m_s(\Pi_s)$
    proves the lemma.
\end{proof}
\subsubsection{Proof of Lemma~\ref{lemma:weighted-mix-gap}}
\label{app_subsubsec:proof_lemma:weighted-mix-gap}
\begin{proof} [Proof of Lemma~\ref{lemma:weighted-mix-gap}]
    Fix $s$.
    Define $Q_s:=\mathcal N(\theta_{s+1},\alpha H_{s+1}^{-1})$, where  $\alpha=3\eta/2$.
    Then, by the variational formula for the mix loss (Lemma~\ref{lemma:mixlog}), we have
    \[
    m_s(\Pi_s)
    \le
    \EE_{\theta\sim Q_s}[\ell_s(\theta)]
    +
    \KL(Q_s\|\Pi_s).
    \]
    Recall the definition of the Bregman divergence
    \[D_{\ell_s}(\theta,\theta_{s+1})
    :=
    \ell_s(\theta)-\ell_s(\theta_{s+1})
    -\langle \nabla \ell_s(\theta_{s+1}),\,\theta-\theta_{s+1}\rangle.\]
    Taking expectation under $Q_s$ and using that $Q_s$ is centered at $\theta_{s+1}$, so that $\EE_{Q_s}[\theta-\theta_{s+1}]=0$, gives
    \[
    \EE_{Q_s}[\ell_s(\theta)]-\ell_s(\theta_{s+1})
    =
    \EE_{Q_s}[D_{\ell_s}(\theta,\theta_{s+1})].
    \]
    Hence, we get
    \begin{align}
    m_s(\Pi_s)-\ell_s(\theta_{s+1})
    \le
    \underbrace{\EE_{Q_s}\bigl[D_{\ell_s}(\theta,\theta_{s+1})\bigr]}_{=:I_{1,s}}
    +
    \underbrace{\KL(Q_s\|\Pi_s)}_{=:I_{2,s}}.
    \label{eq:lemma:weighted-mix-gap_intermid}
    \end{align}
    To bound $I_{1,s}$, we apply
    Lemma~\ref{lemma:local-gb}, which yields
    \begin{align}
        I_{1,s}\le 2\alpha\Bigl(
        \log\det(H_{s+1})-\log\det(A_s)
        \Bigr).
        \label{eq:lemma:weighted-mix-gap_I_1}
    \end{align}
    We next turn to the term $I_{2,s}$.
    Since
    $\Pi_s=\mathcal N(\theta_s,\alpha A_s^{-1})$ and
    $Q_s=\mathcal N(\theta_{s+1},\alpha H_{s+1}^{-1})$,
    the Gaussian KL formula gives
    \[
    I_{2,s}
    =
    \frac12\left[
    \log\frac{\det(H_{s+1})}{\det(A_s)}
    +
    \Tr(A_sH_{s+1}^{-1})
    -d
    +
    \frac{1}{\alpha}\|\theta_{s+1}-\theta_s\|_{A_s}^2
    \right].
    \]
    Since $H_{s+1}\succeq A_s$, we have
    $A_s^{1/2}H_{s+1}^{-1}A_s^{1/2}\preceq I_d$,
    so every eigenvalue of $A_sH_{s+1}^{-1}$ lies in $(0,1]$. 
    Therefore, we get
    \[
    \Tr(A_sH_{s+1}^{-1})-d \le 0.
    \]
    This gives
    \begin{align*}
        I_{2,s}
        &\leq \frac12\left[
        \log\frac{\det(H_{s+1})}{\det(A_s)}
        +
        \frac{1}{\alpha}\|\theta_{s+1}-\theta_s\|_{A_s}^2
        \right]
        \\
        &= \frac12 \Bigl(
            \log\det(H_{s+1})-\log\det(A_s)
        \Bigr) 
        +\frac{1}{2\alpha}\|\theta_{s+1}-\theta_s\|_{A_s}^2
        \numberthis
        \label{eq:lemma:weighted-mix-gap_I_2}
    \end{align*}
    Plugging~\eqref{eq:lemma:weighted-mix-gap_I_1} and~\eqref{eq:lemma:weighted-mix-gap_I_2} into~\eqref{eq:lemma:weighted-mix-gap_intermid}, we obtain
    \begin{align*}
    m_s(\Pi_s)-\ell_s(\theta_{s+1})
    \le
    \left(2\alpha+\frac12\right)\Bigl(             \log\det(H_{s+1})-\log\det(A_s)         \Bigr) 
    +
    \frac{1}{2\alpha}\|\theta_{s+1}-\theta_s\|_{A_s}^2.
    \numberthis \label{eq:lemma:weighted-mix-gap_intermid2}
    \end{align*}
    For simplicity, let
    $L_s:=\log\det(H_s)$.
    Since
    $A_s=\gamma H_s+(1-\gamma)\lambda I_d$,
    concavity of $\log\det(\cdot)$ implies
    \[
    \log\det(A_s)\ge \gamma L_s+(1-\gamma)d\log\lambda.
    \]
    Therefore
    \begin{align}
     \log\det(H_{s+1})-\log\det(A_s)
    =
    L_{s+1}-\log\det(A_s)
    \le
    L_{s+1}-\gamma L_s-(1-\gamma)d\log\lambda.
    \label{eq:lemma:weighted-mix-gap_logdet_gap}
    \end{align}
    Multiplying both sides of~\eqref{eq:lemma:weighted-mix-gap_logdet_gap} by $\gamma^{t-1-s}$ and summing over $s=1,\dots,t-1$ yields
    \begin{align*}
        \sum_{s=1}^{t-1}\gamma^{t-1-s}\Bigl(             \log\det(H_{s+1})-\log\det(A_s)         \Bigr) 
        &\le
        \sum_{s=1}^{t-1}\gamma^{t-1-s}L_{s+1}
        -\sum_{s=1}^{t-1}\gamma^{t-s}L_s
        \\
        &\quad
        -
        (1-\gamma)d\log\lambda \sum_{s=1}^{t-1}\gamma^{t-1-s}.
    \end{align*}
    The first two sums can be telescoped as follows:
    \[
    \sum_{s=1}^{t-1}\gamma^{t-1-s}L_{s+1}
    -
    \sum_{s=1}^{t-1}\gamma^{t-s}L_s
    =
    L_t-\gamma^{t-1}L_1
    =  L_t-\gamma^{t-1} d\log\lambda
    ,
    \]
    where the last equality holds since $L_1=d\log\lambda$.
    Moreover, evaluating the geometric sum gives
    \[
    (1-\gamma)\sum_{s=1}^{t-1}\gamma^{t-1-s}=1-\gamma^{t-1}.
    \]
    Hence, we get
    \begin{align*}
        \sum_{s=1}^{t-1}\gamma^{t-1-s}\Bigl(             \log\det(H_{s+1})-\log\det(A_s)         \Bigr) 
        &\le
        L_t - \gamma^{t-1} d \log \lambda
        - (1 - \gamma^{t-1}) d \log \lambda
        \\
        &=
        L_t-d\log\lambda.
        \numberthis
        \label{eq:lemma:weighted-mix-gap_sum_logdet_gap}
    \end{align*}
    On the other hand, by the definition of $H_t
    =
    \lambda I_d+\sum_{s=1}^{t-1}\gamma^{t-1-s}\frac{\mu'(x_s^\top\theta_{s+1})}{g(\tau)}x_sx_s^\top$,
    we have
    \begin{align*}
        \Tr(H_t)
        &=
        \lambda d+\sum_{s=1}^{t-1}\gamma^{t-1-s}\frac{\mu'(x_s^\top\theta_{s+1})}{g(\tau)}\|x_s\|_2^2
        \\
        &\le
        \lambda d+\frac{k_\mu}{g(\tau)}\sum_{s=1}^{t-1}\gamma^{t-1-s}
        \tag{$\|x\|_2 \leq 1$, Assumption~\ref{assum:boundedness}}
        \\
        &= \lambda d+\frac{k_\mu(1-\gamma^{t-1})}{g(\tau)(1-\gamma)}
        \tag{$\sum_{s=1}^{t-1}\gamma^{t-1-s}=(1-\gamma^{t-1})/(1-\gamma)$}
        .
    \end{align*}
    Therefore, we obtain
    \begin{align*}
        L_t-d\log\lambda
        &= 
        \log\det(H_t)
        -d\log\lambda
        \le
        \log
        \left(\frac{\Tr(H_t)}{d}\right)^d 
        -d\log\lambda
        \tag{AM-GM inequality}
        \\
        &\leq
        d\log\!\left(
        1+\frac{k_\mu(1-\gamma^{t-1})}{\lambda d\,g(\tau)(1-\gamma)}
        \right).
        \numberthis
        \label{eq:lemma:weighted-mix-gap_L_t}
    \end{align*}
    Finally, summing~\eqref{eq:lemma:weighted-mix-gap_intermid2} with weights $\gamma^{t-1-s}$ yields
    \begin{align*}
    \sum_{s=1}^{t-1}&\gamma^{t-1-s}\bigl(m_s(\Pi_s)-\ell_s(\theta_{s+1})\bigr)
    \\
    &\le
    \left(2\alpha+\frac12\right)\sum_{s=1}^{t-1}\gamma^{t-1-s}\Bigl(         \log\det(H_{s+1})-\log\det(A_s)   \Bigr) 
    +
    \frac{1}{2\alpha}\sum_{s=1}^{t-1}\gamma^{t-1-s}\|\theta_{s+1}-\theta_s\|_{A_s}^2 \\
    &\le
    \left(2\alpha+\frac12\right)\bigl(L_t-d\log\lambda\bigr)
    +
    \frac{1}{2\alpha}\sum_{s=1}^{t-1}\gamma^{t-1-s}\|\theta_{s+1}-\theta_s\|_{A_s}^2 
    \tag{Eqn.~\eqref{eq:lemma:weighted-mix-gap_sum_logdet_gap}}
    \\
    &\le
    \left(2\alpha+\frac12\right)
    d\log\!\left(
    1+\frac{k_\mu(1-\gamma^{t-1})}{\lambda d\,g(\tau)(1-\gamma)}
    \right)
    +
    \frac{1}{2\alpha}\sum_{s=1}^{t-1}\gamma^{t-1-s}\|\theta_{s+1}-\theta_s\|_{A_s}^2 
    \tag{Eqn.~\eqref{eq:lemma:weighted-mix-gap_L_t}}
    .
    \end{align*}
    Finally, substituting $\alpha=\frac{3\eta}{2}$ completes the proof.
\end{proof}
\subsubsection{Proof of Lemma~\ref{lemma:drift-loss}}
\label{app_subsubsec:proof_lemma:drift-loss}
\begin{proof}[Proof of Lemma~\ref{lemma:drift-loss}]
    Fix \(t\ge 2\). 
    Using the definition of \(\ell_s\), we obtain
    \begin{align*}
    \ell_s(\theta_t^\star)-\ell_s(\theta_s^\star)
    &=
    \frac{
    m(x_s^\top\theta_t^\star)-r_s\,x_s^\top\theta_t^\star
    -m(x_s^\top\theta_s^\star)+r_s\,x_s^\top\theta_s^\star
    }{g(\tau)} \\
    &=
    \frac{
    m(x_s^\top\theta_t^\star)-m(x_s^\top\theta_s^\star)
    -r_s\,x_s^\top(\theta_t^\star-\theta_s^\star)
    }{g(\tau)} \\
    &=
    \frac{
    m(x_s^\top\theta_t^\star)-m(x_s^\top\theta_s^\star)
    -\mu(x_s^\top\theta_s^\star)\,x_s^\top(\theta_t^\star-\theta_s^\star)
    }{g(\tau)}
    -
    \frac{\eta_s\,x_s^\top(\theta_t^\star-\theta_s^\star)}{g(\tau)}.
    \tag{$r_s = \mu(x_s^\top\theta_s^\star) + \eta_s$}
    \end{align*}
    Therefore, we get
    \[
    \sum_{s=1}^{t-1}\gamma^{t-1-s}
    \bigl(\ell_s(\theta_t^\star)-\ell_s(\theta_s^\star)\bigr)
    =
    B_t^{\mathrm{det}}-M_t,
    \]
    where the deterministic term is defined by
    \[
    B_t^{\mathrm{det}}
    :=
    \frac{1}{g(\tau)}
    \sum_{s=1}^{t-1}\gamma^{t-1-s}
    \Bigl(
    m(x_s^\top\theta_t^\star)-m(x_s^\top\theta_s^\star)
    -\mu(x_s^\top\theta_s^\star)\,x_s^\top(\theta_t^\star-\theta_s^\star)
    \Bigr),
    \]
    and the stochastic term is
    \[
    M_t
    :=
    \frac{1}{g(\tau)}
    \sum_{s=1}^{t-1}\gamma^{t-1-s}\eta_s\,x_s^\top(\theta_t^\star-\theta_s^\star).
    \]

    \emph{i) Deterministic part.}\,
    By Assumption~\ref{assum:boundedness},  both \(x_s^\top\theta_t^\star\) and \(x_s^\top\theta_s^\star\) belong to \([-S,S]\).
    Since 
    \(m'=\mu\),
    \(m''=\mu'\), 
    and Assumption~\ref{assum:bounded_link} gives \(\mu'(z)\le k_\mu\) for all \(z\in[-S,S]\),
    Taylor's theorem implies that
    \begin{align*}
    m(x_s^\top\theta_t^\star)-m(x_s^\top\theta_s^\star)
    -\mu(x_s^\top\theta_s^\star)\,x_s^\top(\theta_t^\star-\theta_s^\star) 
    \le
    \frac{k_\mu}{2}\bigl(x_s^\top(\theta_t^\star-\theta_s^\star)\bigr)^2.
    \end{align*}
    Consequently, we can bound $B_t^{\mathrm{det}}$ as follows:
    \begin{align*}
    B_t^{\mathrm{det}}
    &\le
    \frac{k_\mu}{2g(\tau)}
    \sum_{s=1}^{t-1}\gamma^{t-1-s}
    \bigl(x_s^\top(\theta_t^\star-\theta_s^\star)\bigr)^2
    \\
    &\le \frac{k_\mu}{2g(\tau)}\,\mathcal V_t^2
    \tag{Lemma~\ref{lemma:kernel}}
    .
    \end{align*}

    \emph{ii) Stochastic part.}\,
    For fixed \(t\), let
    \[
    c_{s,t}:=\gamma^{t-1-s}\,x_s^\top(\theta_t^\star-\theta_s^\star),
    \qquad s=1,\dots,t-1.
    \]
    Define the filtration before observing $r_t$ by
    $\mathcal F_t
    :=
    \sigma(x_1,r_1,\dots,x_{t-1},r_{t-1},x_t),$
    and the enlarged filtration
    $\widetilde{\mathcal F}_t
    :=
    \mathcal F_t \vee \sigma(\theta_u^\star:u\ge 1)$.
    Since the path \(\{\theta_u^\star\}_{u=1}^T\) is \textit{oblivious}, i.e.,
    the entire sequence of underlying parameters does not depend on the realized rewards or on the learner's internal randomness,
    and \(x_s\) is
    \(\widetilde{\mathcal F}_s\)-measurable, each \(c_{s,t}\) is
    \(\widetilde{\mathcal F}_s\)-measurable.

    Moreover, because \(r_s\in[0,R]\) almost surely (Assumption~\ref{assum:boundedness}) and
    \(\mu(x_s^\top\theta_s^\star)=\mathbb E[r_s | \widetilde{\mathcal F}_s]\in[0,R]\), we have
    $\eta_s \in [-\mu(x_s^\top\theta_s^\star),\,R-\mu(x_s^\top\theta_s^\star)]$
    almost surely.
    Thus, conditioned on \(\widetilde{\mathcal F}_s\), the centered random variable \(\eta_s\)
    lies in an interval of length \(R\). 
    Hence, by Hoeffding's lemma, for every \(\bar{\lambda}>0\), we get
    \[
    \mathbb E\!\left[
    \exp(-\bar{\lambda} c_{s,t}\eta_s)
    \,\middle|\,
    \widetilde{\mathcal F}_s
    \right]
    \le
    \exp\!\left(\frac{\bar{\lambda}^2 c_{s,t}^2 R^2}{8}\right).
    \]

    Let
    $Y_t:=\sum_{s=1}^{t-1}c_{s,t}\eta_s.$
    Iterating conditional expectations from \(s=t-1\) down to \(s=1\) yields
    \[
    \mathbb E\!\left[e^{-\bar{\lambda} Y_t}\right]
    \le
    \exp\!\left(
    \frac{\bar{\lambda}^2 R^2}{8}
    \sum_{s=1}^{t-1}c_{s,t}^2
    \right).
    \]
    Therefore, for every \(u>0\) and every \(\bar{\lambda}>0\), Chernoff's bound gives
    \begin{align*}
    \mathbb P(-M_t\ge u)
    &=
    \mathbb P\!\left(
    -\sum_{s=1}^{t-1}c_{s,t}\eta_s
    \ge
    g(\tau)u
    \right) \\
    &\le
    \exp\!\left(-\bar{\lambda} g(\tau)u\right)
    \mathbb E\!\left[e^{-\bar{\lambda} Y_t}\right] 
    \tag{Chernoff's bound}
    \\
    &\le
    \exp\!\left(
    -\bar{\lambda} g(\tau)u
    +
    \frac{\bar{\lambda}^2 R^2}{8}
    \sum_{s=1}^{t-1}c_{s,t}^2
    \right).
    \end{align*}
    Optimizing over \(\bar{\lambda}>0\) gives
    $\bar{\lambda}^\star
    =
    \frac{4g(\tau)u}{R^2\sum_{s=1}^{t-1}c_{s,t}^2},$
    and thus
    \[
    \mathbb P(-M_t\ge u)
    \le
    \exp\!\left(
    -\frac{2g(\tau)^2u^2}
    {R^2\sum_{s=1}^{t-1}c_{s,t}^2}
    \right).
    \]
    Hence, with probability at least \(1-\delta\), we have
    \begin{align*}
    -M_t
    &\le
    \frac{R}{g(\tau)}
    \sqrt{
    \frac12\log\frac{1}{\delta}
    \sum_{s=1}^{t-1}c_{s,t}^2
    }
    \\
    &= \frac{R}{g(\tau)}
    \sqrt{
    \frac12\log\frac{1}{\delta}
    \sum_{s=1}^{t-1}\gamma^{2(t-1-s)}
    \bigl(x_s^\top(\theta_t^\star-\theta_s^\star)\bigr)^2
    }
    \\
    &\leq \frac{R}{g(\tau)}
    \sqrt{\frac12\log\frac{1}{\delta}}\,
    \mathcal V_t.
    \tag{Lemma~\ref{lemma:kernel}}
    \end{align*}
    Applying Young's inequality  $ab\le \frac{\varepsilon}{2}a^2+\frac{1}{2\varepsilon}b^2$ 
    with 
    $a=\mathcal V_t$,
    $b=\frac{R}{g(\tau)}\sqrt{\frac12\log\frac{1}{\delta}}$,
    and $\varepsilon=\frac{k_\mu}{g(\tau)}$,
    we obtain
    \[
    -M_t
    \le
    \frac{k_\mu}{2g(\tau)}\mathcal V_t^2
    +
    \frac{R^2}{4g(\tau)k_\mu}\log\frac{1}{\delta}.
    \]

    Combining the bounds on \(B_t^{\mathrm{det}}\) and \(-M_t\), we conclude that,
    with probability at least \(1-\delta\),
    \[
    \sum_{s=1}^{t-1}\gamma^{t-1-s}
    \bigl(\ell_s(\theta_t^\star)-\ell_s(\theta_s^\star)\bigr)
    \le
    \frac{k_\mu}{g(\tau)}\,\mathcal V_t^2
    +
    \frac{R^2}{4g(\tau)k_\mu}\log\frac{1}{\delta}.
    \]
    This completes the proof.
\end{proof}

\subsection{Auxiliary Lemmas for Theorem~\ref{thm:DOMD_confidence}} 
\label{app_subsec:aux_thm_confidence}
\begin{lemma}[Self-concordance control]
    \label{lemma:self-concor-control}
    Under Assumption~\ref{assum:boundedness}, 
    for every $v,z_1,z_2\in\RR$,
    \begin{equation}
    \label{eq:self-concor-control-lower}
    \int_0^1 (1-s)\mu'\bigl(v+s(z_1-z_2)\bigr)\, \dd s
    \ge
    \frac{\mu'(v)}{2+R|z_1-z_2|},
    \end{equation}
    and
    \begin{equation}
    \label{eq:self-concor-control-upper}
    \int_0^1 (1-s)\mu'\bigl(v+s(z_1-z_2)\bigr)\, \dd s
    \le
    \exp\!\left(\frac{R^2(z_1-z_2)^2}{4}\right)\mu'(v).
    \end{equation}
\end{lemma}
\begin{proof}[Proof of Lemma~\ref{lemma:self-concor-control}]
    Under the bounded rewards assumption (Assumption~\ref{assum:boundedness}),
    Proposition~\ref{prop:self-concordance} implies that
    $|\mu''(z)| \le R\,\mu'(z)$ for all $z \in \RR$.
    Equivalently,
    \[
    \left|\frac{d}{du}\log\mu'(u)\right|
    =
    \left|\frac{\mu''(u)}{\mu'(u)}\right|
    \le R.
    \]
    Hence, for any $s\in[0,1]$,
    \begin{align*}
        \log \mu'\bigl(v+s(z_1-z_2)\bigr)-\log \mu'(v)
        &=
        \int_0^1
        \frac{d}{d w}\log \mu'\bigl(v+w s(z_1-z_2)\bigr)\,\dd w
        \\
        &=
        \int_0^1
        \frac{\mu''\bigl(v+w s(z_1-z_2)\bigr)}{\mu'\bigl(v+ ws(z_1-z_2)\bigr)}\, s(z_1-z_2)\, \dd w.
    \end{align*}
    Therefore,
    \[
    \left|\log \mu'\bigl(v+s(z_1-z_2)\bigr)-\log \mu'(v)\right|
    \le
    \int_0^1 R\,s|z_1-z_2|\, \dd w
    =
    Rs|z_1-z_2|,
    \]
    which implies
    \[
    e^{-Rs|z_1-z_2|}\mu'(v)
    \le
    \mu'\bigl(v+s(z_1-z_2)\bigr)
    \le
    e^{Rs|z_1-z_2|}\mu'(v).
    \]

    For the lower bound, let $u:=R|z_1-z_2|$.
    Then
    \begin{align*}
        \int_0^1 (1-s)\mu'\bigl(v+s(z_1-z_2)\bigr)\,\dd s
    &\ge
    \mu'(v)\int_0^1 (1-s)e^{-u s}\, \dd s \\
    &=
    \mu'(v)\frac{u-1+e^{-u}}{u^2}
    \qquad (u>0),
    \end{align*}
    where the last equality follows from a direct calculation.

    Hence it suffices to show
    \[
    \frac{u-1+e^{-u}}{u^2}\ge \frac{1}{2+u}
    \qquad \forall u > 0.
    \]
    Define
    $q(u):=(2+u)(u-1+e^{-u})-u^2.$
    Then $q(0)=0$, and
    \[
    q'(u)=1-(1+u)e^{-u}\ge 0
    \qquad \forall u > 0,
    \]
    because $e^u > 1+u$.
    Therefore $q(u) > 0$ for all $u > 0$, which proves \eqref{eq:self-concor-control-lower}.

    For the upper bound, 
    using
    $\mu'\bigl(v+s(z_1-z_2)\bigr)\le e^{us}\mu'(v),$
    we obtain
    \[
    \int_0^1 (1-s)\mu'\bigl(v+s(z_1-z_2)\bigr)\,\dd s
    \le
    \mu'(v)\int_0^1 (1-s)e^{us}\,\dd s.
    \]
    Let \(B\sim \mathrm{Beta}(1,2)\), whose density is \(2(1-s)\mathbf 1\{0\le s\le 1\}\). Then
    \[
    \int_0^1 (1-s)e^{us}\,\dd s=\frac12 \EE[e^{uB}].
    \]
    Since \(B\in[0,1]\), Hoeffding's lemma yields
    \begin{equation}
        \label{eq:beta_hoeffding}
        \EE[e^{u(B-\EE B)}]\le e^{u^2/8}.    
    \end{equation}
    Using \(\EE B=1/3\), we get
    \begin{align*}
        \int_0^1 (1-s)e^{us}\,\dd s
        &=
        \frac12 \EE[e^{uB}]
        =
        \frac12 e^{u \EE B }\EE[e^{u(B-\EE B)}]
        =
        \frac12 e^{u/3}\EE[e^{u(B-\EE B)}]
        \\
        &\le
        \frac12 \exp\!\left(\frac{u}{3}+\frac{u^2}{8}\right).
        \tag{Eqn.~\eqref{eq:beta_hoeffding}}
    \end{align*}
    
    Moreover,
    $\frac{u^2}{8}+\frac{2}{9}-\frac{u}{3}
    =
    \frac{(3u-4)^2}{72}\ge 0,$
    so
    $\frac{u}{3}\le \frac{u^2}{8}+\frac{2}{9}$.
    Hence
    \begin{align*}
    \int_0^1 (1-s)e^{us}\,\dd s
    \leq
     \frac12 \exp\!\left(\frac{u}{3}+\frac{u^2}{8}\right)
    \le
    \frac12 e^{2/9}e^{u^2/4}
    \le
    e^{u^2/4},
    \tag{\(e^{2/9}/2<1\)}
    \end{align*}
    which proves \eqref{eq:self-concor-control-upper}.
\end{proof}
\begin{lemma}[Gaussian exponential moment]
\label{lemma:gauss-exp}
    For any $\alpha, c >0$,
    let $Z\sim\mathcal N(0,\alpha H^{-1})$, where $H\succeq \lambda I_d$.
    If
    $\lambda \ge \frac{64\alpha dc}{7},$
    then, we have
    \[
    \EE\!\left[e^{c\|Z\|_2^2}\right]\le \frac43.
    \]
\end{lemma}
\begin{proof} [Proof of Lemma~\ref{lemma:gauss-exp}]
    Let $\xi\sim\mathcal N(0,I_d)$ and write $Z=\sqrt{\alpha}\,H^{-1/2}\xi$.
    Since $H^{-1}\preceq \lambda^{-1}I_d$, we have pointwise
    \[
    \|Z\|_2^2
    =
    \alpha\,\xi^\top H^{-1}\xi
    \le
    \frac{\alpha}{\lambda}\|\xi\|_2^2.
    \]
    Therefore
    \begin{align*}
    \EE\!\left[e^{c\|Z\|_2^2}\right]
    \le
    \EE\!\left[e^{\frac{\alpha c}{\lambda}\|\xi\|_2^2}\right].
    \numberthis
    \label{eq:lemma:gauss-exp_upper}
    \end{align*}
    Now write $\xi=(\xi_1,\dots,\xi_d)$ with $\xi_i \stackrel{\mathrm{i.i.d.}}{\sim}\mathcal N(0,1)$. 
    Then, since
    $\|\xi\|_2^2=\sum_{i=1}^d \xi_i^2,$
    by defining $h := \alpha c/\lambda$, we get
    \begin{align*}
    \EE\!\left[e^{h\|\xi\|_2^2}\right]
    &=
    \EE\!\left[\exp\!\left(h\sum_{i=1}^d \xi_i^2\right)\right]
    =
    \prod_{i=1}^d \EE\!\left[e^{h \xi_i^2}\right].
    \end{align*}
    For each $i$, we have
    \begin{align*}
    \EE\!\left[e^{h \xi_i^2}\right]
    &=
    \frac{1}{\sqrt{2\pi}}\int_{\mathbb R} \exp\!\left(h x^2-\frac{x^2}{2}\right)\,dx \\
    &=
    \frac{1}{\sqrt{2\pi}}\int_{\mathbb R} \exp\!\left(-\frac{1-2h}{2}x^2\right)\,dx \\
    &=
    \frac{1}{\sqrt{1-2h}}.
    \tag{$\frac{\alpha c}{\lambda} \leq \frac{7}{64 d} < \frac{1}{2}$}
    \end{align*}
    Hence, we get
    \[
    \EE\!\left[e^{h\|\xi\|_2^2}\right]
    =
    \prod_{i=1}^d \EE\!\left[e^{h \xi_i^2}\right]
    =
    (1-2h)^{-d/2}.
    \]
    Plugging this into \eqref{eq:lemma:gauss-exp_upper} and substituting $h=\alpha c/\lambda$, we obtain
    \[
    \EE\!\left[e^{c\|Z\|_2^2}\right]
    \le
    \left(1-\frac{2\alpha c}{\lambda}\right)^{-d/2}.
    \]
    Finally, our assumption $\lambda \ge \frac{64\alpha dc}{7}$ implies
    $\frac{2\alpha c}{\lambda}\le \frac{7}{32d} < 1$.
    Since the map $x\mapsto (1-x)^{-d/2}$ is increasing on $(0,1)$, it follows that
    \[
    \left(1-\frac{2\alpha c}{\lambda}\right)^{-d/2}
    \le
    \left(1-\frac{7}{32d}\right)^{-d/2}.
    \]
    Now use the elementary bound $-\log(1-x)\le \frac{x}{1-x}$ for $x\in(0,1)$. With $x=\frac{7}{32d}$, we obtain
    \begin{align*}
    \left(1-\frac{7}{32d}\right)^{-d/2}
    &=
    \exp\!\left(-\frac d2 \log\!\left(1-\frac{7}{32d}\right)\right) 
    \le
    \exp\!\left(\frac d2 \cdot \frac{7/(32d)}{1-7/(32d)}\right) \\
    &=
    \exp\!\left(\frac{7/64}{1-7/(32d)}\right).
    \end{align*}
    Finally, since $d\ge 1$,
    we get
    $1-\frac{7}{32d}\ge 1-\frac{7}{32}=\frac{25}{32}$,
    and therefore
    \[
    \frac{7/64}{1-7/(32d)}
    \le
    \frac{7/64}{25/32}
    =
    \frac{7}{50}.
    \]
    Hence, we have
    \[
    \left(1-\frac{2\alpha c}{\lambda}\right)^{-d/2}
    \le
    \exp\!\left(\frac{7/64}{1-7/(32d)}\right)
    \le
    e^{7/50}
    <
    \frac43,
    \]
    which proves the lemma.
\end{proof}
\begin{lemma}
\label{lemma:local-gb}
    Suppose \(\|x_s\|_2\le 1\)
    and $\lambda \geq 32 \alpha d R^2 /7$.
    For each \(s \ge 1\), let
    $Q_s := \mathcal N(\theta_{s+1}, \alpha H_{s+1}^{-1})$
    be the Gaussian distribution centered at \(\theta_{s+1}\) with covariance matrix \(\alpha H_{s+1}^{-1}\).
    Also define the Bregman divergence of \(\ell_s\) at \(\theta_{s+1}\) by
    $D_{\ell_s}(\theta,\theta_{s+1})
    :=
    \ell_s(\theta)-\ell_s(\theta_{s+1})
    -\langle \nabla \ell_s(\theta_{s+1}),\,\theta-\theta_{s+1}\rangle.$
    Then, we have
    \[
    \EE_{\theta\sim Q_s}\!\left[D_{\ell_s}(\theta,\theta_{s+1})\right]
    \le
    2\alpha\Bigl(
    \log\det(H_{s+1})-\log\det(A_s)
    \Bigr).
    \]
\end{lemma}

\begin{proof}[Proof of Lemma~\ref{lemma:local-gb}]
    For convenience, define
    $\omega_s:=\frac{\mu'(x_s^\top\theta_{s+1})}{g(\tau)}$, so that 
    \[
    \nabla^2\ell_s(\theta_{s+1})
    =
    \frac{\mu'(x_s^\top\theta_{s+1})}{g(\tau)}x_sx_s^\top
    =
    \omega_s x_sx_s^\top,
    \]
    By the integral form of the Bregman divergence, we can write $D_{\ell_s}(\theta,\theta_{s+1})$ as follows:
    \[
    D_{\ell_s}(\theta,\theta_{s+1})
    =
    \frac{\bigl(x_s^\top(\theta-\theta_{s+1})\bigr)^2}{g(\tau)}
    \int_0^1 (1-\nu)\mu'\bigl(x_s^\top\theta_{s+1}+\nu x_s^\top(\theta-\theta_{s+1})\bigr)\,\dd\nu,
    \]
    Then, by applying Lemma~\ref{lemma:self-concor-control}, we obtain
    \[
    D_{\ell_s}(\theta,\theta_{s+1})
    \le
    \exp\!\left(\frac{R^2\bigl(x_s^\top(\theta-\theta_{s+1})\bigr)^2}{4}\right)
    \frac{\mu'(x_s^\top\theta_{s+1})}{g(\tau)}
    \bigl(x_s^\top(\theta-\theta_{s+1})\bigr)^2.
    \]
    Since
    \[
    \| \theta-\theta_{s+1}\|_{\nabla^2\ell_s(\theta_{s+1})}^2
    =
    (\theta-\theta_{s+1})^\top \nabla^2\ell_s(\theta_{s+1})(\theta-\theta_{s+1})
    =
    \frac{\mu'(x_s^\top \theta_{s+1})}{g(\tau)}
    \bigl(x_s^\top(\theta-\theta_{s+1})\bigr)^2,
    \]
    it follows that
    \begin{align*}
        D_{\ell_s}(\theta,\theta_{s+1})
        &\le
        \exp\!\left(\frac{R^2\bigl(x_s^\top(\theta-\theta_{s+1})\bigr)^2}{4}\right)
        \|\theta-\theta_{s+1}\|_{\nabla^2\ell_s(\theta_{s+1})}^2
        \\
        &\le 
        \exp\!\left(\frac{R^2\|\theta-\theta_{s+1}\|_2^2}{4}\right)
        \|\theta-\theta_{s+1}\|_{\nabla^2\ell_s(\theta_{s+1})}^2
        \numberthis \label{eq:lemma:local-gb_breg}
        .
    \end{align*}
    where the last inequality holds because \(\|x_s\|_2\le 1\).
    Now let
    $W_{s+1}:=\theta-\theta_{s+1}\sim\mathcal N(0,\alpha H_{s+1}^{-1})$.
    By taking expectations on both sides of \eqref{eq:lemma:local-gb_breg} and applying Cauchy--Schwarz, we obtain
    \begin{align}
    \EE_{Q_s}\bigl[D_{\ell_s}(\theta,\theta_{s+1})\bigr]
    \le
    \sqrt{
    \EE\!\left[e^{R^2\|W_{s+1}\|_2^2/2}\right]\;
    \EE\!\left[\|W_{s+1}\|_{\nabla^2\ell_s(\theta_{s+1})}^4\right]
    }.
    \label{eq:lemma:local-gb_upper}
    \end{align}
    Because \(H_{s+1}\succeq \lambda I_d\) and \(\lambda\ge 32\alpha dR^2/7 = 64 \alpha d c /7\), Lemma~\ref{lemma:gauss-exp} with \(c=R^2/2\) yields
    \begin{align}
    \EE\!\left[e^{R^2\|W_{s+1}\|_2^2/2}\right]\le \frac43.
    \label{eq:lemma:local-gb_upper_1}
    \end{align}
    
    It remains to bound the fourth moment.
    Let
    $\Sigma_s:=\alpha H_{s+1}^{-1}$ and
    $M_s:=\Sigma_s^{1/2} \nabla^2\ell_s(\theta_{s+1}) \Sigma_s^{1/2}$.
    Since \(W_{s+1}\sim \mathcal N(0,\Sigma_s)\), we may write
    \[
    W_{s+1}=\Sigma_s^{1/2}\xi,
    \qquad
    \xi\sim\mathcal N(0,I_d).
    \]
    Therefore,
    \[
    W_{s+1}^\top \nabla^2\ell_s(\theta_{s+1}) W_{s+1}
    =
    \xi^\top M_s \xi.
    \]
    
    Now diagonalize \(M_s\) as
    \[
    M_s=U\Lambda U^\top,
    \qquad
    \Lambda=\mathrm{diag}(\rho_1,\dots,\rho_d),
    \]
    where \(U\) is an orthogonal matrix
    and $\rho_1,\dots,\rho_d \geq 0$ are the eigenvalues of $M_s$.
    Since \(U^\top \xi\sim\mathcal N(0,I_d)\), letting
    $\zeta:=U^\top \xi$,
    we obtain
    \[
    W_{s+1}^\top \nabla^2\ell_s(\theta_{s+1}) W_{s+1}
    =
    \xi^\top M_s \xi
    =
    \zeta^\top \Lambda \zeta
    =
    \sum_{i=1}^d \rho_i \zeta_i^2,
    \]
    where \(\zeta_1,\dots,\zeta_d\) are i.i.d. standard Gaussian random variables. 
    Hence
    \[
    \|W_{s+1}\|_{\nabla^2\ell_s(\theta_{s+1})}^4 
    =
    (\xi^\top M_s \xi)^2
    =
    \left(\sum_{i=1}^d \rho_i \zeta_i^2\right)^2
    =
    \sum_{i=1}^d \rho_i^2 \zeta_i^4
    +
    \sum_{i\neq j}\rho_i\rho_j \zeta_i^2\zeta_j^2.
    \]
    Taking expectations and using
    $\EE[\zeta_i^4]=3$ and
    $\EE[\zeta_i^2\zeta_j^2]=1$ for
    $i\neq j$,
    we get
    \begin{align*}
    \EE \! \left[ \|W_{s+1}\|_{\nabla^2\ell_s(\theta_{s+1})}^4  \right]
    &=
    \EE\left[(\xi^\top M_s \xi)^2 \right]
    =
    3\sum_{i=1}^d \rho_i^2
    +
    \sum_{i\neq j}\rho_i\rho_j \\
    &=
    2\sum_{i=1}^d \rho_i^2
    +
    \left(\sum_{i=1}^d \rho_i\right)^2
    \\
    &= 2\Tr(M_s^2)+\Tr(M_s)^2
    .
    \end{align*}
    Recall that \(M_s=\Sigma_s^{1/2}\nabla^2\ell_s(\theta_{s+1})\Sigma_s^{1/2}\).
    Since
        $\nabla^2\ell_s(\theta_{s+1})
        =
        \frac{\mu'(x_s^\top\theta_{s+1})}{g(\tau)}x_sx_s^\top,$
    the matrix \(M_s\) is positive semidefinite and has rank at most one. Hence,
    if \(\lambda_s(M_s)\) denotes its unique possibly nonzero eigenvalue, then
    \[
        \Tr(M_s)=\lambda_s(M_s),
        \qquad
        \Tr(M_s^2)=\lambda_s(M_s)^2,
    \]
    and therefore
    \[
        \Tr(M_s)^2=\Tr(M_s^2).
    \]
    Therefore, we obtain
    \begin{align*}
        \EE \! \left[ \|W_{s+1}\|_{\nabla^2\ell_s(\theta_{s+1})}^4  \right]
        &=
        2\Tr(M_s^2)+\Tr(M_s)^2
        \\
        &=
        3\Tr(M_s)^2
        \tag{\(M_s\) is rank-one}
        \\
        &=
        3\Tr\!\left(
            \Sigma_s
            \nabla^2\ell_s(\theta_{s+1})
        \right)^2
        \tag{cyclicity of trace}
        \\
        &=
        3\alpha^2
        \Tr\!\left(
            H_{s+1}^{-1}
            \nabla^2\ell_s(\theta_{s+1})
        \right)^2
        \numberthis
        \label{eq:lemma:local-gb_upper_2}
        .
    \end{align*}
    Substituting~\eqref{eq:lemma:local-gb_upper_1} and~\eqref{eq:lemma:local-gb_upper_2}) into~\eqref{eq:lemma:local-gb_upper}, we obtain
    \begin{align}
    \EE_{Q_s}\bigl[D_{\ell_s}(\theta,\theta_{s+1})\bigr]
    \le
    2\alpha\,\Tr(H_{s+1}^{-1}\nabla^2\ell_s(\theta_{s+1})).
    \label{eq:lemma:local-gb_intermid}
    \end{align}
    
    Now write
    $H_{s+1}=A_s+\omega_s x_sx_s^\top$ and
    $u_s:=x_s^\top A_s^{-1}x_s$.
    By the Sherman--Morrison formula,
    \begin{align*}
    \Tr(H_{s+1}^{-1}\nabla^2\ell_s(\theta_{s+1}))
    &=
    \omega_s x_s^\top H_{s+1}^{-1}x_s \\
    &=
    \omega_s\!\left(
    x_s^\top A_s^{-1}x_s
    -
    \frac{\omega_s\,x_s^\top A_s^{-1}x_s\,x_s^\top A_s^{-1}x_s}
    {1+\omega_s x_s^\top A_s^{-1}x_s}
    \right) \\
    &=
    \omega_s\!\left(u_s-\frac{\omega_s u_s^2}{1+\omega_s u_s}\right)
    =
    \frac{\omega_s u_s}{1+\omega_s u_s}.
    \numberthis
    \label{eq:lemma:local-gb_trace}
    \end{align*}
    
    On the other hand, we have
    \begin{align*}
    \log\det(H_{s+1})-\log\det(A_s) 
    &=
    \log\det(A_s+\omega_s x_sx_s^\top)-\log\det(A_s) \\
    &=
    \log\Bigl(\det(A_s)\bigl(1+\omega_s x_s^\top A_s^{-1}x_s\bigr)\Bigr)-\log\det(A_s) \\
    &=
    \log(1+\omega_s u_s).
    \numberthis
    \label{eq:lemma:local-gb_logdet}
    \end{align*}
    
    Since \(\log(1+y)\ge y/(1+y)\) for every \(y>-1\), it follows that
    \begin{align*}
    \Tr(H_{s+1}^{-1}\nabla^2\ell_s(\theta_{s+1}))
    &=
    \frac{\omega_s u_s}{1+\omega_s u_s}
    \tag{Eqn.~\eqref{eq:lemma:local-gb_trace}}
    \\
    &\le
    \log(1+\omega_s u_s)
    \tag{\(\log(1+y)\ge y/(1+y)\)}
    \\
    &=
    \log\det(H_{s+1})-\log\det(A_s). 
    \tag{Eqn.~\eqref{eq:lemma:local-gb_logdet}}
    \end{align*}
    Combining this with~\eqref{eq:lemma:local-gb_intermid} completes the proof.
\end{proof}
\begin{lemma} [Variational formula for mix loss, Lemma 12 of \citealt{zhang2025generalized}]
\label{lemma:mixlog}
    Let \(P\) be a probability distribution defined over \(\mathbb{R}^d\) and \(\Delta\) be the set of all measurable distributions. For any loss function \(\ell:\mathbb{R}^d\to\mathbb{R}\), we have
    \begin{equation*}
    -\frac{1}{\alpha}\log\!\left(\mathbb{E}_{\theta\sim P}\!\left[e^{-\alpha \ell(\theta)}\right]\right)
    =
    \mathbb{E}_{\theta\sim P_\star}[\ell(\theta)]
    +
    \frac{1}{\alpha}\mathrm{KL}(P_\star\|P),
    \end{equation*}
    where
    $P_\star=\argmin_{P'\in\Delta}\,
    \mathbb{E}_{\theta\sim P'}[\ell(\theta)]
    +
    \frac{1}{\alpha}\mathrm{KL}(P'\|P)$
    is the optimal solution. Furthermore, for any distribution \(Q\) defined over \(\mathbb{R}^d\), we have
    \begin{equation*}
    -\frac{1}{\alpha}\log\!\left(\mathbb{E}_{\theta\sim P}\!\left[e^{-\alpha \ell(\theta)}\right]\right)
    =
    \mathbb{E}_{\theta\sim Q}[\ell(\theta)]
    +
    \frac{1}{\alpha}\mathrm{KL}(Q\|P)
    -
    \frac{1}{\alpha}\mathrm{KL}(Q\|P_\star).
    \end{equation*}
\end{lemma}

\begin{lemma}[Kernel compression]
\label{lemma:kernel}
    For \(t\ge 2\), define
    $\mathcal V_t
        :=
        \sum_{p=1}^{t-1} \gamma^{(t-1-p)/2}\sqrt{\frac{1-\gamma^p}{1-\gamma}} \|\theta_{p+1}^\star-\theta_p^\star\|_2$.
    Assume that $\| x_s \|_2 \leq 1$ for all $s \in [T]$.
    Then, we have
    \[
    \sum_{s=1}^{t-1}\gamma^{t-1-s}
    \left( x_s^\top(\theta_t^\star-\theta_s^\star) \right)^2
    \le
    \mathcal V_t^2
    \qquad
    \text{and} \qquad
    \sum_{s=1}^{t-1}\gamma^{2(t-1-s)}
    \left( x_s^\top(\theta_t^\star-\theta_s^\star) \right)^2
    \le
    \mathcal V_t^2.
    \]
\end{lemma}
\begin{proof}[Proof of Lemma~\ref{lemma:kernel}]
    Fix \(t\ge 2\).
    For simplicity, define
    $\zeta_{s,t}:=x_s^\top(\theta_t^\star-\theta_s^\star)
    =
    x_s^\top\sum_{p=s}^{t-1}
    \left(\theta_{p+1}^\star-\theta_p^\star \right)$.
    Then, for each \(s\in\{1,\dots,t-1\}\), we have
    \[
    \zeta_{s,t}
    =
    x_s^\top \sum_{p=s}^{t-1}(\theta_{p+1}^\star-\theta_p^\star)
    =
    \sum_{p=1}^{t-1}\mathbbm{1}\{s\le p\}\,
    x_s^\top(\theta_{p+1}^\star-\theta_p^\star).
    \]
    Define
    $u_t
    :=
    \bigl(\gamma^{(t-1-s)/2}\zeta_{s,t}\bigr)_{s=1}^{t-1}
    \in \mathbb R^{t-1},$
    and, for each \(p\in\{1,\dots,t-1\}\),
    \[
    v_{p,t}
    :=
    \Bigl(
    \mathbbm{1}\{s\le p\}\,
    \gamma^{(t-1-s)/2}\,
    x_s^\top(\theta_{p+1}^\star-\theta_p^\star)
    \Bigr)_{s=1}^{t-1}
    \in \mathbb R^{t-1}.
    \]
    Then, by the above decomposition of \(\zeta_{s,t}\), we have
    \[
    u_t=\sum_{p=1}^{t-1} v_{p,t}.
    \]
    
    We now bound \(\|v_{p,t}\|_2\). 
    For each \(p\), we get
    \begin{align*}
    \|v_{p,t}\|_2^2
    &=
    \sum_{s=1}^{t-1}
    \mathbbm{1}\{s\le p\}\,
    \gamma^{t-1-s}
    \bigl(x_s^\top(\theta_{p+1}^\star-\theta_p^\star)\bigr)^2 \\
    &=
    \sum_{s=1}^{p}
    \gamma^{t-1-s}
    \bigl(x_s^\top(\theta_{p+1}^\star-\theta_p^\star)\bigr)^2 \\
    &\le
    \sum_{s=1}^{p}
    \gamma^{t-1-s}\,
    \|x_s\|_2^2\,
    \|\theta_{p+1}^\star-\theta_p^\star\|_2^2 \\
    &\le
    \|\theta_{p+1}^\star-\theta_p^\star\|_2^2
    \sum_{s=1}^{p}\gamma^{t-1-s}.
    \tag{$\|x_s \|_2^2 \leq 1$}
    \end{align*}
    Since
    \[
    \sum_{s=1}^{p}\gamma^{t-1-s}
    =
    \gamma^{t-1-p}\sum_{r=0}^{p-1}\gamma^r
    =
    \gamma^{t-1-p}\frac{1-\gamma^p}{1-\gamma},
    \]
    it follows that
    \[
    \|v_{p,t}\|_2
    \le
    \gamma^{(t-1-p)/2}
    \sqrt{\frac{1-\gamma^p}{1-\gamma}}\,
    \|\theta_{p+1}^\star-\theta_p^\star\|_2.
    \]
    
    Next, by the triangle inequality in \(\mathbb R^{t-1}\), we obtain
    \begin{align*}
        \|u_t\|_2
        =
        \left\|\sum_{p=1}^{t-1}v_{p,t}\right\|_2
        \le
        \sum_{p=1}^{t-1}\|v_{p,t}\|_2
        \le
        \sum_{p=1}^{t-1}
        \gamma^{(t-1-p)/2}
        \sqrt{\frac{1-\gamma^p}{1-\gamma}}\,
        \|\theta_{p+1}^\star-\theta_p^\star\|_2
        =
        \mathcal V_t.
    \end{align*}
    Therefore, by the definition of \(u_t\), we get
    \[
    \|u_t\|_2^2 
    = \sum_{s=1}^{t-1}\gamma^{t-1-s}\zeta_{s,t}^2
    \le \mathcal V_t^2,
    \]
    which proves the first inequality.
    
    It remains to prove the second inequality. Since \(\gamma\in(0,1)\), for every \(s\le t-1\),
    \[
    \gamma^{2(t-1-s)}\le \gamma^{t-1-s}.
    \]
    Therefore,
    \[
    \sum_{s=1}^{t-1}\gamma^{2(t-1-s)}\zeta_{s,t}^2
    \le
    \sum_{s=1}^{t-1}\gamma^{t-1-s}\zeta_{s,t}^2
    \le
    \mathcal V_t^2.
    \]
    This completes the proof.
\end{proof}

\section{Proof of Theorem~\ref{thm:regret}}
\label{app_sec:proof_thm:regret}
In this section, we provide the proof of Theorem~\ref{thm:regret}. 
\subsection{Main Proof of Theorem~\ref{thm:regret}}
\label{app_subsec:main_proof_thm:regret}
We begin by stating useful lemmas that will be used throughout the proof. 
First, define the high-probability event
\begin{equation}
    \Ecal
    :=
    \bigcap_{t\ge 2}
    \left\{
    \theta_t^\star \in \Ccal_t(\delta)
    \right\}.
    \label{eq:good_event}
\end{equation}
By Theorem~\ref{thm:DOMD_confidence}, this event satisfies
\[
\Pr(\Ecal)\ge 1-\delta.
\]
On \(\Ecal\), the confidence bound in parameter space directly translates into a corresponding prediction bound for the mean reward, which will be used repeatedly in the regret analysis.
\begin{lemma}[Prediction bound under the discounted-OMD confidence event]
\label{lemma:DOMD_prediction}
    Under the same parameter choices as in Theorem~\ref{thm:DOMD_confidence}, and under Assumptions~\ref{assum:boundedness} and~\ref{assum:bounded_link}, on the event \(\Ecal\), the following holds for every \(t \ge 2\) and every feasible action \(x \in \Xcal_t\):
    \[
    \left|x^\top(\theta_t-\theta_t^\star)\right|
        \le
        \beta_t(\delta)\|x\|_{H_t^{-1}}
        +
        \frac{\iota}{\sqrt{\lambda}}\mathcal V_t.
    \]
\end{lemma}
The proof is provided in Appendix~\ref{app_subsubsec:proof_lemma:DOMD_prediction}.

Lemma~\ref{lemma:elliptical_discount} is a discounted version of the standard elliptical potential lemma, closely related to Lemma~8 of~\citet{faury2021regret}; the main difference is that their analysis is based on $V_t=\lambda I_d+\sum_{s=1}^{t-1}\gamma^{t-1-s}x_sx_s^\top$, whereas ours uses a Hessian-based matrix $
H_t
=
\lambda I_d+\sum_{s=1}^{t-1}\gamma^{t-1-s}\nabla^2 \ell_s(\theta_{s+1})$.
\begin{lemma}[Discounted elliptical potential]
    \label{lemma:elliptical_discount}
    Assume $\lambda\ge c_\mu/g(\tau)$. 
    Then for every $T\ge 1$, we have
    \begin{align*}
    \sum_{t=1}^T \|x_t\|_{H_t^{-1}}^2
    \le
    \frac{2g(\tau)}{c_\mu} \left[  d\log\!\left(1+\frac{k_\mu}{\lambda d\,g(\tau)(1-\gamma)}\right)  +Td\log\frac{1}{\gamma} \right].
    \end{align*}
\end{lemma}
The proof is provided in Appendix~\ref{app_subsubsec:proof_lemma:elliptical_discount}.

Lemma~\ref{lemma:kernel_sum} provides a bound on the cumulative drift kernel 
\(\mathcal V_t\), similarly to prior work such as~\citet{wang2023revisiting}.
The geometric discounting suppresses the influence of past parameter changes, so the total contribution is controlled by the path length \(P_T\).
\begin{lemma}[Summing the drift kernel]
\label{lemma:kernel_sum}
Let $\mathcal V_t
:=
\sum_{p=1}^{t-1} \gamma^{(t-1-p)/2}\sqrt{\frac{1-\gamma^p}{1-\gamma}} \|\theta_{p+1}^\star-\theta_p^\star\|_2$.
Then, for every horizon $T\ge 2$, we have
    \begin{equation*}
    \sum_{t=2}^{T}\mathcal V_t
    \le
    \frac{1+\sqrt\gamma}{(1-\gamma)^{3/2}}\,P_T
    \le
    \frac{2}{(1-\gamma)^{3/2}}\,P_T.
    \end{equation*}
\end{lemma}
The proof is provided in Appendix~\ref{app_subsubsec:proof_lemma:kernel_sum}.

Now, we are ready to provide the proof of Theorem~\ref{thm:regret}.
\begin{proof} [Proof of Theorem~\ref{thm:regret}]
    Throughout the proof, we condition on the event
    \[
    \Ecal
    :=
    \bigcap_{t\ge 2}
    \left\{
    \theta_t^\star\in\Ccal_t(\delta)
    \right\},
    \]
    defined in~\eqref{eq:good_event},
    which holds with probability at least $1-\delta$ by Theorem~\ref{thm:DOMD_confidence}.
    
    Let
        $R_t
        :=
        \mu((x_t^\star)^\top\theta_t^\star)
        -
        \mu(x_t^\top\theta_t^\star)$
    be the instantaneous pseudo-regret. 
    Then, by Lemma~\ref{lemma:DOMD_prediction}, with probability at least \(1-\delta\), for every \(x\in\mathcal X_t\), we have
    \begin{align*}
        R_t 
        &=
        \mu((x_t^\star)^\top\theta_t^\star)
        -
        \mu(x_t^\top\theta_t^\star)
        \\
        &\leq k_\mu \left( (x_t^\star - x_t)^\top\theta_t^\star \right)
        \tag{\(k_\mu\)-Lipschitz continuity of \(\mu\)}
        \\
        &\leq
         k_\mu \left( (x_t^\star)^\top\theta_t  +\beta_t(\delta)\|x_t^\star\|_{H_t^{-1}}
            +\frac{\iota}{\sqrt{\lambda}}\mathcal V_t
            - 
                \left(
                x_t^\top\theta_t
                    -\beta_t(\delta)\|x_t\|_{H_t^{-1}}
                    -\frac{\iota}{\sqrt{\lambda}}\mathcal V_t
                \right)
            \right) 
        \tag{Lemma~\ref{lemma:DOMD_prediction}}
        \\
        &\leq 
        2k_\mu\beta_t(\delta)\|x_t\|_{H_t^{-1}}
        +
        2\frac{k_\mu\iota}{\sqrt{\lambda}}\mathcal V_t.
        \tag{action-selection rule in Eqn.~\eqref{eq:discounted_ucb_rule}}
    \end{align*}
    %
    %
    Summing over $t=1,\dots,T$ yields
    \begin{align*}
        \Regret
        &\le
        R_1
        +
        2k_\mu\sum_{t=2}^{T}\beta_t(\delta)\|x_t\|_{H_t^{-1}}
        +
        2\frac{k_\mu\iota}{\sqrt\lambda}\sum_{t=2}^{T}\mathcal V_t
        \\
        &\le R_1
        +
        2k_\mu \beta_T(\delta) \sum_{t=2}^{T}\|x_t\|_{H_t^{-1}}
        +
        2\frac{k_\mu\iota}{\sqrt\lambda}\sum_{t=2}^{T}\mathcal V_t
        \tag{$\beta_t(\delta)$ is nondecreasing}
        \\
        &\le 
        R_1
        +
        2k_\mu \beta_T(\delta)\sqrt{T\sum_{t=2}^{T}\|x_t\|_{H_t^{-1}}^2} 
        +
        2\frac{k_\mu\iota}{\sqrt\lambda}\sum_{t=2}^{T}\mathcal V_t
        \tag{Cauchy-Schwarz inequality}
        \\
        &\le
        R_1
        +
        2k_\mu \beta_T(\delta)\sqrt{
            T\frac{2g(\tau)}{c_\mu}
            \left(
            d\log\!\left(1+\frac{k_\mu}{\lambda d\,g(\tau)(1-\gamma)}\right)
            +
            Td\log\frac{1}{\gamma} \right)
            }
        +
        2\frac{k_\mu\iota}{\sqrt\lambda}\sum_{t=2}^{T}\mathcal V_t
        \tag{Lemma~\ref{lemma:elliptical_discount}}
        \\
        &\le
        R_1
        +
        2k_\mu\beta_T(\delta)\sqrt{
        T\frac{2g(\tau)}{c_\mu}
        \left[
        d\log\!\left(1+\frac{k_\mu}{\lambda d\,g(\tau)(1-\gamma)}\right)
        +
        Td\log\frac{1}{\gamma}
        \right]}
        \\
        &\quad+
        \frac{4k_\mu\iota}{\sqrt\lambda}\frac{P_T}{(1-\gamma)^{3/2}}
        \tag{Lemma~\ref{lemma:kernel_sum}}
        \\
        &\le
        2k_\mu S
        +
        2k_\mu\beta_T(\delta)\sqrt{
        T\frac{2g(\tau)}{c_\mu}
        \left[
        d\log\!\left(1+\frac{k_\mu}{\lambda d\,g(\tau)(1-\gamma)}\right)
        +
        Td\log\frac{1}{\gamma}
        \right]}
        \\
        &\quad+
        \frac{4k_\mu\iota}{\sqrt\lambda}\frac{P_T}{(1-\gamma)^{3/2}}
        \numberthis \label{eq:regret_upper}
        ,
    \end{align*}
    where the last inequality holds because $R_1
    \le
    k_\mu(\|x_1^\star\|_2+\|x_1\|_2)\|\theta_1^\star\|_2
    \le
    2k_\mu S$.
    
    By Theorem~\ref{thm:DOMD_confidence}, under
    \[
    \lambda \ge \max\left\{
    \frac{6\eta Rk_\mu S}{g(\tau)},
    \frac{32\alpha dR^2}{7},
    \frac{c_\mu}{g(\tau)}
    \right\},
    \quad
    \eta = 1+RS,
    \quad
    \alpha = \frac{3\eta}{2},
    \quad
    \iota = \sqrt{\frac{2\eta k_\mu}{g(\tau)}},
    \]
    we have
    \begin{align}
    \beta_T(\delta)=\BigOTilde(\sqrt d)
    \qquad
    \text{and}
    \qquad
    \frac{k_\mu\iota}{\sqrt\lambda}
    =
    \BigOTilde\!\left(\frac{k_\mu^{3/2}}{\sqrt{c_\mu}}\right).
    \label{eq:regret_beta}
    \end{align}
    Moreover, 
    since \(\gamma\in[1/T,1)\), we have
    \[
    \log\frac1\gamma
    =
    \begin{cases}
    \BigO(1-\gamma), & \text{if }\gamma\in[1/2,1),\\[3pt]
    \BigO(\log T)=\BigOTilde(1), & \text{if }\gamma\in[1/T,1/2),
    \end{cases}
    \]
    and hence, in both cases,
    \[
    T\sqrt{\log(1/\gamma)}
    =
    \BigOTilde\bigl(T\sqrt{1-\gamma}\bigr).
    \]
    Therefore, we have
    \begin{align}
    \sqrt{
    \frac{2g(\tau)T}{c_\mu}
    \left[
    d\log\!\left(1+\frac{k_\mu}{\lambda d\,g(\tau)(1-\gamma)}\right)
    +
    Td\log\frac{1}{\gamma}
    \right]}
    =
    \BigOTilde\!\left(
    \sqrt{\frac{Td}{c_\mu}}
    +
    T\sqrt{\frac{d(1-\gamma)}{c_\mu}}
    \right).
    \label{eq:regret_sqrt}
    \end{align}
    Plugging~\eqref{eq:regret_beta} and~\eqref{eq:regret_sqrt} into~\eqref{eq:regret_upper}, we obtain
    \[
    \Regret
    =
    \BigOTilde\!\left(
    \frac{k_\mu}{\sqrt{c_\mu}}\,d\sqrt T
    +
    \frac{k_\mu}{\sqrt{c_\mu}}\,dT\sqrt{1-\gamma}
    +
    \frac{k_\mu^{3/2}}{\sqrt{c_\mu}}\frac{P_T}{(1-\gamma)^{3/2}}
    \right).
    \]
    This proves the first bound.

    Now we prove the second bound.
    For the tuned choice, let
    \[
    \gamma
    =
    1-
    \min\left\{
    1-\frac1T,\;
    \max\left\{
    \frac1T,\;
    \sqrt{\frac{\sqrt{k_\mu}P_T}{dT}}
    \right\}
    \right\}.
    \]
    Then \(\gamma\in[1/T,1)\), and therefore the first bound applies.
    
    \paragraph{Case 1.} If
    $\frac{d}{\sqrt{k_\mu}\,T}
    \le
    P_T
    \le
    \frac{dT}{\sqrt{k_\mu}}
    \left(1-\frac1T\right)^2,$
    then
    $1-\gamma
    =
    \sqrt{\frac{\sqrt{k_\mu}P_T}{dT}}.$
    This implies that
    \[
    \frac{k_\mu}{\sqrt{c_\mu}}\,dT\sqrt{1-\gamma}
    =
    \BigOTilde\!\left(
    \frac{k_\mu^{9/8}}{\sqrt{c_\mu}}\,d^{3/4}P_T^{1/4}T^{3/4}
    \right),
    \]
    and
    \[
    \frac{k_\mu^{3/2}}{\sqrt{c_\mu}}\frac{P_T}{(1-\gamma)^{3/2}}
    =
    \BigOTilde\!\left(
    \frac{k_\mu^{9/8}}{\sqrt{c_\mu}}\,d^{3/4}P_T^{1/4}T^{3/4}
    \right).
    \]
    Moreover, under
    $P_T\ge \frac{d}{\sqrt{k_\mu}\,T}$,
    the first term
    $\frac{k_\mu}{\sqrt{c_\mu}}\,d\sqrt T$
    is dominated by the same quantity. 
    Therefore, we get
    \[
    \Regret
    =
    \BigOTilde\!\left(
    \frac{k_\mu^{9/8}}{\sqrt{c_\mu}}\,d^{3/4}P_T^{1/4}T^{3/4}
    \right).
    \]
    
    \paragraph{Case 2.}
    If
    $0\le P_T<\frac{d}{\sqrt{k_\mu}\,T},$
    then
    $1-\gamma=\frac1T.$
    In this case, we have
    \[
    \frac{k_\mu}{\sqrt{c_\mu}}\,dT\sqrt{1-\gamma}
    =
    \BigOTilde\!\left(
    \frac{k_\mu}{\sqrt{c_\mu}}\,d\sqrt T
    \right),
    \]
    and
    \[
    \frac{k_\mu^{3/2}}{\sqrt{c_\mu}}\frac{P_T}{(1-\gamma)^{3/2}}
    =
    \BigOTilde\!\left(
    \frac{k_\mu^{3/2}}{\sqrt{c_\mu}}\,P_T T^{3/2}
    \right)
    <
    \BigOTilde\!\left(
    \frac{k_\mu}{\sqrt{c_\mu}}\,d\sqrt T
    \right).
    \]
    Hence, we obtain
    \[
    \Regret
    \le
    \BigOTilde\!\left(
    \frac{k_\mu}{\sqrt{c_\mu}}\,d\sqrt T
    \right).
    \]
    
    \paragraph{Case 3.}
    If
    $P_T>
    \frac{dT}{\sqrt{k_\mu}}
    \left(1-\frac1T\right)^2,$
    then
    $1-\gamma=1-\frac1T,$
    that is, 
    $\gamma=\frac1T.$
    In this case, we have
    \[
    \frac{k_\mu}{\sqrt{c_\mu}}\,dT\sqrt{1-\gamma}
    =
    \BigOTilde\!\left(
    \frac{k_\mu}{\sqrt{c_\mu}}\,dT
    \right),
    \]
    and
    \[
    \frac{k_\mu^{3/2}}{\sqrt{c_\mu}}\frac{P_T}{(1-\gamma)^{3/2}}
    =
    \BigOTilde\!\left(
    \frac{k_\mu^{3/2}}{\sqrt{c_\mu}}\,P_T
    \right).
    \]
    Moreover, the first term
    $\frac{k_\mu}{\sqrt{c_\mu}}\,d\sqrt T$
    is dominated by
    $\frac{k_\mu}{\sqrt{c_\mu}}\,dT.$
    Since \(T\ge 2\), we have
    $\left(1-\frac1T\right)^2\ge \frac14,$
    and thus
    \[
    P_T>
    \frac{dT}{\sqrt{k_\mu}}
    \left(1-\frac1T\right)^2
    \ge
    \frac{dT}{4\sqrt{k_\mu}}.
    \]
    Therefore, the first term can be bounded as
    \[
    \frac{k_\mu}{\sqrt{c_\mu}}\,d\sqrt T
    \le
    \frac{k_\mu}{\sqrt{c_\mu}}\,dT
    \le
    4\frac{k_\mu^{3/2}}{\sqrt{c_\mu}}\,P_T.
    \]
    Hence, all three terms are dominated by
    $\BigOTilde\!\left(
    \frac{k_\mu^{3/2}}{\sqrt{c_\mu}}\,P_T
    \right),$
    and we obtain
    \[
    \Regret
    \le
    \BigOTilde\!\left(
    \frac{k_\mu^{3/2}}{\sqrt{c_\mu}}\,P_T
    \right).
    \]
    This completes the proof of Theorem~\ref{thm:regret}.
\end{proof}

\begin{remark}[On the clipped choice of \(\gamma\)]
\label{remark:clip_gamma}
    The clipped definition of \(\gamma\) is necessary because our generic regret bound is proved only for admissible \(\gamma \in [1/T,1)\), whereas the unconstrained minimizer of the simplified upper bound need not belong to this interval.
    Earlier works often stated only the interior regimes and implicitly assumed that the resulting choice of \(\gamma\) was feasible.
    For example, \citet{faury2021regret} effectively restrict to a sublinear-\(P_T\) regime in the proof, which allows them to exclude the large-\(P_T\) case and thus ensures that the tuned choice \(\gamma = 1 - \left(\frac{P_T}{dT}\right)^{2/3}\) remains admissible.
    Similarly, \citet{wang2023revisiting} first derives generic bounds for admissible \(\gamma \in (1/T,1)\), but then states tuned two-case corollaries without an explicit feasibility check.
    A similar issue also appears in \citet{russac2020algorithms,russac2021self} for the piecewise-stationary setting, where the tuning implicitly relies on a sublinear-\(\Gamma_T\) regime so that the proposed choice \(\gamma \approx 1 - \left(\frac{\Gamma_T}{dT}\right)^{2/3}\) remains admissible.
    By contrast, our three-case bound is fully rigorous: the clipping step enforces admissibility explicitly, so the resulting bound applies uniformly over all parameter regimes. If the same feasibility check were made explicit in the earlier analyses, their tuned bounds would likewise take an analogous multi-case form, rather than covering only the interior one- or two-case regimes.
\end{remark}

\subsection{Proofs of Lemmas for Theorem~\ref{thm:regret}}
\label{app_subsec:proof_lemmas_thm:regret}
\subsubsection{Proof of Lemma~\ref{lemma:DOMD_prediction}}
\label{app_subsubsec:proof_lemma:DOMD_prediction}
\begin{proof}[Proof of Lemma~\ref{lemma:DOMD_prediction}]
    Fix any \(t\ge 2\) and any feasible action \(x\in\Xcal_t\). On the event \(\Ecal\), we have
    $\theta_t^\star\in \Ccal_t(\delta),$
    and hence, by the definition of \(\Ccal_t(\delta)\),
    we have
    \[
    \|\theta_t^\star-\theta_t\|_{H_t}
    \le
    \beta_t(\delta)+\iota\mathcal V_t.
    \]
    Therefore, by Cauchy--Schwarz inequality,
    \begin{align*}
    |x^\top(\theta_t^\star-\theta_t)|
    &\le
    \|x\|_{H_t^{-1}}\|\theta_t^\star-\theta_t\|_{H_t} \\
    &\le
    \beta_t(\delta)\|x\|_{H_t^{-1}}
    +
    \iota\mathcal V_t\,\|x\|_{H_t^{-1}}
    \\
    &\le 
    \beta_t(\delta)\|x\|_{H_t^{-1}}
    +
    \frac{\iota}{\sqrt{\lambda}}\mathcal V_t,
    \end{align*}
    where the last inequality holds because
    \begin{align*}
    \|x\|_{H_t^{-1}}
    \le
    \frac{\|x\|_2}{\sqrt{\lambda}}
    \le
    \frac{1}{\sqrt{\lambda}}.
    \tag{Assumption~\ref{assum:boundedness}}
    \end{align*}
    %
    This completes the proof.
\end{proof}
\subsubsection{Proof of Lemma~\ref{lemma:elliptical_discount}}
\label{app_subsubsec:proof_lemma:elliptical_discount}
\begin{proof} [Proof of Lemma~\ref{lemma:elliptical_discount}]
    Since $A_t
    = \gamma H_t + (1-\gamma)\lambda I_d
    \preceq \gamma H_t + (1-\gamma)H_t
    = H_t$, inverse monotonicity on positive definite matrices gives
    $A_t^{-1}\succeq H_t^{-1}$, which 
    implies
    $\|x_t\|_{A_t^{-1}}^2 \ge \|x_t\|_{H_t^{-1}}^2$.
    
    Now, we write
    \[
    \nabla^2\ell_t(\theta_{t+1})
    =
    \frac{\mu'(x_t^\top\theta_{t+1})}{g(\tau)}x_tx_t^\top
    =:h_t x_tx_t^\top,
    \]
    where Assumption~\ref{assum:bounded_link} ensures that 
    $\frac{c_\mu}{g(\tau)}\le h_t\le \frac{k_\mu}{g(\tau)}$.

    Then, we have
    \begin{align*}
    \log\detn(H_{t+1})-\log\detn(A_t) 
    &= \log\detn(A_t + h_t x_t x_t^\top)-\log\detn(A_t) \\
    &=
    \log\detn\bigl(I_d+h_tA_t^{-1/2}x_tx_t^\top A_t^{-1/2}\bigr) \\
    &=
    \log\bigl(1+h_t\|x_t\|_{A_t^{-1}}^2\bigr)
    \\
    &\ge \log\!\left(1+\frac{c_\mu}{g(\tau)}\|x_t\|_{H_t^{-1}}^2\right).
    \tag{$h_t\ge c_\mu/g(\tau)$ and $\|x_t\|_{A_t^{-1}}^2\ge \|x_t\|_{H_t^{-1}}^2$}
    \end{align*}
    Since $\|x_t\|_{H_t^{-1}}^2\le 1/\lambda$ and $\lambda\ge c_\mu/g(\tau)$, we have $0\le (c_\mu/g(\tau))\|x_t\|_{H_t^{-1}}^2\le 1$.
    Using $\log(1+u)\ge u/2$ for $u\in[0,1]$, we get
    \[
    \log\detn(H_{t+1})-\log\detn(A_t) \ge \frac{c_\mu}{2g(\tau)}\|x_t\|_{H_t^{-1}}^2,
    \]
    which implies
    \begin{align}
    \sum_{t=1}^T \|x_t\|_{H_t^{-1}}^2\le \frac{2g(\tau)}{c_\mu} 
    \sum_{t=1}^T\big( 
         \log\detn(H_{t+1})-\log\detn(A_t)
    \big).
    \label{eq:sum_elliptical}
    \end{align}
    Because $A_t=\gamma H_t+(1-\gamma)\lambda I_d\succeq \gamma H_t$, monotonicity of $\log\detn$ gives
    \[
    \log\detn(A_t)\ge \log\detn(\gamma H_t)=d\log\gamma+\log\detn(H_t).
    \]
    Hence, we get
    \[
         \log\detn(H_{t+1})-\log\detn(A_t)
    \le
    \log\detn(H_{t+1})-\log\detn(H_t)-d\log\gamma.
    \]
    Summing from $t=1$ to $T$ yields
    \begin{align*}
    \sum_{t=1}^T
    \big( 
         \log\detn(H_{t+1})-\log\detn(A_t)
    \big)
    &\le
    \log\detn(H_{T+1})-\log\detn(H_1)+Td\log\frac1\gamma
    \\
    &= \log\detn(H_{T+1})-d\log\lambda+Td\log\frac1\gamma.
    \tag{$H_1=\lambda I_d$}
    \end{align*}
    Also, we have
    \[
    H_{T+1}
    =
    \lambda I_d+\sum_{s=1}^{T}\gamma^{T-s}\frac{\mu'(x_s^\top\theta_{s+1})}{g(\tau)}x_sx_s^\top
    \preceq
    \lambda I_d+\frac{k_\mu}{g(\tau)}\sum_{s=1}^{T}\gamma^{T-s}x_sx_s^\top.
    \]
    Using $\|x_s\|_2\le 1$, we get
    \[
    \Tr(H_{T+1})
    \le
    \lambda d+\frac{k_\mu}{g(\tau)}\sum_{s=1}^{T}\gamma^{T-s}
    \le
    \lambda d+\frac{k_\mu}{g(\tau)(1-\gamma)}.
    \]
    By AM--GM, we obtain
    \[
    \detn(H_{T+1})
    \le
    \left(\frac{\Tr(H_{T+1})}{d}\right)^d
    \le
    \left(\lambda+\frac{k_\mu}{d\,g(\tau)(1-\gamma)}\right)^d.
    \]
    which implies.
    \begin{align*}
        \sum_{t=1}^T
        \big( 
             \log\detn(H_{t+1})-\log\detn(A_t)
        \big)
        &\leq \log\detn(H_{T+1})-d\log\lambda+Td\log\frac1\gamma
        \\
        &\leq d\log\!\left(1+\frac{k_\mu}{\lambda d\,g(\tau)(1-\gamma)}\right)+Td\log\frac1\gamma.
        \numberthis \label{eq:sum_logdet}
    \end{align*}
    Substituting \eqref{eq:sum_logdet} into \eqref{eq:sum_elliptical} completes the proof.
\end{proof}
\subsubsection{Proof of Lemma~\ref{lemma:kernel_sum}}
\label{app_subsubsec:proof_lemma:kernel_sum}
\begin{proof} [Proof of Lemma~\ref{lemma:kernel_sum}]
    By the definition of $\mathcal V_t$,
    we have
    \begin{align*}
    \sum_{t=2}^{T}\mathcal V_t
    &\le
    \sum_{t=2}^{T}\sum_{p=1}^{t-1}\frac{\gamma^{(t-1-p)/2}}{\sqrt{1-\gamma}}\,\|\theta_{p+1}^\star-\theta_p^\star\|_2 \\
    &=
    \sum_{p=1}^{T-1}\|\theta_{p+1}^\star-\theta_p^\star\|_2\sum_{t=p+1}^{T}\frac{\gamma^{(t-1-p)/2}}{\sqrt{1-\gamma}} \\
    &=
    \frac{1}{\sqrt{1-\gamma}}\sum_{p=1}^{T-1}\|\theta_{p+1}^\star-\theta_p^\star\|_2\sum_{k=0}^{T-1-p}\gamma^{k/2} \\
    &\le
    \frac{1}{\sqrt{1-\gamma}}\sum_{p=1}^{T-1}\|\theta_{p+1}^\star-\theta_p^\star\|_2\cdot \frac{1}{1-\sqrt\gamma} \\
    &=
    \frac{1}{\sqrt{1-\gamma}(1-\sqrt\gamma)}\,P_T
    \\
    &= \frac{1+\sqrt\gamma}{(1-\gamma)^{3/2}}\,P_T
    ,
    \tag{$1-\sqrt\gamma=\frac{1-\gamma}{1+\sqrt\gamma}$}
    \end{align*}
    which proves the first inequality.
    Since $1+\sqrt\gamma\le 2$, the second inequality follows.
\end{proof}

\section{Proof of Theorem~\ref{thm:regret_piecewise}}
\label{app_sec:proof_thm:regret_piecewise}
In this section, we prove Theorem~\ref{thm:regret_piecewise}.
The overall argument closely follows that of Appendix~\ref{app_sec:proof_thm:regret};
the only additional ingredient is a sharper bound on \(\sum_t \mathcal V_t\) over rounds that are sufficiently far from the most recent change point.
\subsection{Main Proof of Theorem~\ref{thm:regret_piecewise}}
\label{app_subsec:main_proof_thm:regret_piecewise}
As in the proof of Theorem~\ref{thm:regret}, we condition on the good event
\[
\Ecal
:=
\bigcap_{t\ge 2}
\left\{
\theta_t^\star\in\Ccal_t(\delta)
\right\},
\]
which, by Theorem~\ref{thm:DOMD_confidence}, holds with probability at least \(\Pr(\Ecal)\ge 1-\delta\).
We further use Lemma~\ref{lemma:DOMD_prediction} and Lemma~\ref{lemma:elliptical_discount} from Appendix~\ref{app_sec:proof_thm:regret}.

For each integer \(D \ge 1\), let \(\mathcal G_D\) denote the set of rounds \(t \in [T]\) such that the environment parameter has not changed during the previous \(D\) rounds. 
More precisely,
\[
\mathcal G_D
:=
\big\{
t\in[T]:
\theta_s^\star=\theta_t^\star
\ \text{for all } s\in\{\max\{t-D, 1\},\ldots,t-1\}
\big\}.
\]
Equivalently, \(t \in \mathcal G_D\) if the parameter value at round \(t\) coincides with its value at every round in the most recent \(D\)-step history prior to \(t\) (or from round \(1\) onward when \(t \le D\)).

We first quantify how many rounds fail to belong to \(\mathcal G_D\).
Under piecewise stationarity, a round is excluded from \(\mathcal G_D\) only if a changepoint occurred within its most recent \(D\)-step history.
Therefore, each changepoint can affect only a length-\(D\) window of subsequent rounds, which yields the following bound on the number of such ``bad'' rounds.
\begin{lemma}[Bad-round count under piecewise stationarity]
\label{lemma:pw_bad_rounds}
    For every integer \(D\ge 1\), we have
    \[
    |[T]\setminus \mathcal G_D|
    \le
    \Gamma_T D.
    \]
\end{lemma}
The proof is deferred to Appendix~\ref{app_subsubsec:proof_lemma:pw_bad_rounds}.

We next control the cumulative drift term over the complementary set of ``good'' rounds \(\mathcal G_D\).
On these rounds, the parameter has remained unchanged throughout the most recent \(D\) rounds, so any contribution to the drift quantity \(\mathcal V_t\) must come from changes that occurred strictly earlier.
Because the discounted kernel downweights older rounds geometrically, those remote changes are suppressed by a factor of order \(\gamma^{D/2}\), leading to the following estimate.
\begin{lemma}[Piecewise drift-kernel sum]
\label{lemma:pw_kernel_sum}
Under the piecewise-stationary assumption, for every integer \(D\ge 1\),
    \[
    \sum_{t\in\mathcal G_D\cap\{2,\dots,T\}}\mathcal V_t
    \le
    \frac{2\gamma^{D/2}}{(1-\gamma)^{3/2}}\,P_T
    \le
    \frac{4S\Gamma_T\,\gamma^{D/2}}{(1-\gamma)^{3/2}}.
    \]
\end{lemma}
The proof is deferred to Appendix~\ref{app_subsubsec:proof_lemma:pw_kernel_sum}.

Now, we are ready to provide the proof of Theorem~\ref{thm:regret_piecewise}.
\begin{proof}[Proof of Theorem~\ref{thm:regret_piecewise}]
    We work on the event
    $\Ecal
    :=
    \bigcap_{t\ge 2}
    \left\{
    \theta_t^\star\in\Ccal_t(\delta)
    \right\},$
    defined in~\eqref{eq:good_event},
    which satisfies \(\Pr(\Ecal)\ge 1-\delta\) by Theorem~\ref{thm:DOMD_confidence}.
    
    Let
    $R_t
    :=
    \mu((x_t^\star)^\top\theta_t^\star)-\mu(x_t^\top\theta_t^\star)$ and
    $x_t^\star\in\argmax_{x\in\Xcal_t}\mu(x^\top\theta_t^\star),$
    so that
    $\Regret=\sum_{t=1}^T R_t.$
    Fix any integer \(D\ge 1\). We decompose the regret as
    \[
    \Regret
    =
    R_1
    +
    \sum_{t\in\mathcal G_D\cap\{2,\dots,T\}}R_t
    +
    \sum_{t\in([T]\setminus\mathcal G_D)\cap\{2,\dots,T\}}R_t.
    \]
    
    For every \(t\ge 2\), the same argument as in the proof of Theorem~\ref{thm:regret}, using Lemma~\ref{lemma:DOMD_prediction} 
    and our action-selection rule in~\eqref{eq:discounted_ucb_rule} (i.e., UCB optimality of $x_t$), yields
    \[
    R_t
    \le
    2k_\mu\beta_t(\delta)\|x_t\|_{H_t^{-1}}
    +
    2\frac{k_\mu\iota}{\sqrt{\lambda}}\mathcal V_t.
    \]
    Hence, we have
    \begin{align*}
        \sum_{t\in\mathcal G_D\cap\{2,\dots,T\}}R_t
        &\le
        2k_\mu
        \sum_{t\in\mathcal G_D\cap\{2,\dots,T\}}
        \beta_t(\delta)\|x_t\|_{H_t^{-1}}
        +
        2\frac{k_\mu\iota}{\sqrt{\lambda}}
        \sum_{t\in\mathcal G_D\cap\{2,\dots,T\}}\mathcal V_t \\
        &\le
        2k_\mu \beta_T(\delta)\sum_{t=2}^{T}\|x_t\|_{H_t^{-1}}
        +
        2\frac{k_\mu\iota}{\sqrt{\lambda}}
        \sum_{t\in\mathcal G_D\cap\{2,\dots,T\}}\mathcal V_t \\
        &\le
        2k_\mu \beta_T(\delta)
        \sqrt{T\sum_{t=2}^{T}\|x_t\|_{H_t^{-1}}^2}
        +
        2\frac{k_\mu\iota}{\sqrt{\lambda}}
        \sum_{t\in\mathcal G_D\cap\{2,\dots,T\}}\mathcal V_t \\
        &\le
        2k_\mu \beta_T(\delta)\sqrt{\frac{2g(\tau)T}{c_\mu} \left[ d\log\!\left( 1+\frac{k_\mu}{\lambda d\,g(\tau)(1-\gamma)} \right) + Td\log\frac1\gamma \right]}
        \\
        &\quad+
        \frac{8Sk_\mu\iota}{\sqrt{\lambda}}
        \frac{\Gamma_T\,\gamma^{D/2}}{(1-\gamma)^{3/2}}.
        \tag{Lemma~\ref{lemma:elliptical_discount}
            and Lemma~\ref{lemma:pw_kernel_sum}}
    \end{align*}
    On the other hand, by Assumption~\ref{assum:boundedness},
    every conditional mean reward belongs to \([0,R]\), and hence
    $0\le R_t\le R$,
    for all
    $t\in[T]$.
    Therefore, we get
    \begin{align*}
    R_1
    +
    \sum_{t\in([T]\setminus\mathcal G_D)\cap\{2,\dots,T\}}R_t
    &\le
    R\bigl(1+|[T]\setminus\mathcal G_D|\bigr)
    \\
    &\le
    R(\Gamma_T D+1).
    \tag{Lemma~\ref{lemma:pw_bad_rounds}}
    \end{align*}
    Putting the above two bounds together, we obtain that, for every integer \(D\ge 1\),
    \begin{align*}
    \Regret
    &\le
    R(\Gamma_T D+1)
    +
    2k_\mu \beta_T(\delta)\sqrt{\frac{2g(\tau)T}{c_\mu} \left[ d\log\!\left( 1+\frac{k_\mu}{\lambda d\,g(\tau)(1-\gamma)} \right) + Td\log\frac1\gamma \right]}
    \\
    &\quad+
    \frac{8Sk_\mu\iota}{\sqrt{\lambda}}
    \frac{\Gamma_T\,\gamma^{D/2}}{(1-\gamma)^{3/2}}.
    \numberthis \label{eq:piecewise_regret_decomposition}
    \end{align*}

    Now
    we define
    $C_{\mathrm{pw}} := \max \left\{ 
        e,\, \frac{8Sk_\mu\iota}{\sqrt{\lambda}}
    \right\}$
    and
    $D_\gamma
    :=
    \left\lceil
    \frac{2\log\bigl(C_{\mathrm{pw}}  T/(1-\gamma)\bigr)}{\log(1/\gamma)}
    \right\rceil
    .$
    Then, we have
    \[
    \frac{D_\gamma}{2}\log\frac1\gamma
    \ge
    \log\frac{C_{\mathrm{pw}} T}{1-\gamma}.
    \]
    Equivalently,
    \[
    \gamma^{D_\gamma/2}
    =
    \exp\!\left(-\frac{D_\gamma}{2}\log\frac1\gamma\right)
    \le
    \frac{1-\gamma}{C_{\mathrm{pw}} T}
    \le \frac{(1-\gamma)\sqrt{\lambda}}{8Sk_\mu\iota T}
    .
    \]
    Moreover, since \(\gamma\in[1/T,1)\), we have
    \[
    \log\frac1\gamma
    =
    \begin{cases}
        \Theta(1-\gamma), & \text{if }\gamma\in[1/2,1),\\[3pt]
        \Omega(1), & \text{if }\gamma\in[1/T,1/2),
    \end{cases}
    \]
    and therefore
    \[
    D_\gamma
    =
    \left\lceil
    \frac{2\log \bigl (C_{\mathrm{pw}} T/(1-\gamma)\bigr)}{\log(1/\gamma)}
    \right\rceil
    =
    \begin{cases}
    \BigOTilde\!\left(\dfrac{1}{1-\gamma}\right),
    & \text{if }\gamma\in[1/2,1),\\[8pt]
    \BigO(\log T)=\BigOTilde\!\left(\dfrac{1}{1-\gamma}\right),
    & \text{if }\gamma\in[1/T,1/2).
    \end{cases}
    \]
    Hence, in both cases,
    \[
    D_\gamma
    =
    \BigOTilde\!\left(\frac{1}{1-\gamma}\right).
    \]

    Hence, substituting \(D=D_\gamma\) into~\eqref{eq:piecewise_regret_decomposition} yields
    \begin{align*}
        \Regret
        &\le
        R(\Gamma_T D_\gamma+1)
        +
        2k_\mu \beta_T(\delta)\sqrt{\frac{2g(\tau)T}{c_\mu} \left[ d\log\!\left( 1+\frac{k_\mu}{\lambda d\,g(\tau)(1-\gamma)} \right) + Td\log\frac1\gamma \right]}
        \\
        &\quad+
        \frac{8Sk_\mu\iota}{\sqrt{\lambda}}
        \frac{\Gamma_T\,\gamma^{D_\gamma/2}}{(1-\gamma)^{3/2}}.
        \numberthis
        \label{eq:piecewise_regret_decomposition2}
    \end{align*}
    The middle term is handled exactly as in the proof of Theorem~\ref{thm:regret}:
    \begin{align*}
        2k_\mu \beta_T(\delta)&\sqrt{\frac{2g(\tau)T}{c_\mu} \left[ d\log\!\left( 1+\frac{k_\mu}{\lambda d\,g(\tau)(1-\gamma)} \right) + Td\log\frac1\gamma \right]}
        \\
        &\qquad\qquad\qquad\qquad\qquad\qquad\qquad=
        \BigOTilde\!\left(
        \frac{k_\mu}{\sqrt{c_\mu}}\,d\sqrt T
        +
        \frac{k_\mu}{\sqrt{c_\mu}}\,dT\sqrt{1-\gamma}
        \right).
        \numberthis
        \label{eq:piecewise_regret_1}
    \end{align*}
    Also, we have
    \begin{align*}
        R(\Gamma_T D_\gamma+1)
        =
        \BigOTilde\!\left(
        1+\frac{\Gamma_T}{1-\gamma}
        \right),
        \numberthis
        \label{eq:piecewise_regret_2}
    \end{align*}
    and, using \(\gamma^{D_\gamma/2}\le \frac{(1-\gamma)\sqrt{\lambda}}{8Sk_\mu\iota T}\), we get
    \begin{align*}
        \frac{8Sk_\mu\iota}{\sqrt{\lambda}}
        \frac{\Gamma_T\,\gamma^{D_\gamma/2}}{(1-\gamma)^{3/2}}
        \le
        \frac{\Gamma_T}{T\sqrt{1-\gamma}}
        \le
        \BigOTilde\!\left(
        \frac{\Gamma_T}{1-\gamma}
        \right).
        \numberthis
        \label{eq:piecewise_regret_3}
        \end{align*}
    Plugging~\eqref{eq:piecewise_regret_1},~\eqref{eq:piecewise_regret_2}, and~\eqref{eq:piecewise_regret_3} into~\eqref{eq:piecewise_regret_decomposition2}, we obtain
    \[
    \Regret
    \le
    \BigOTilde\!\left(
    \frac{k_\mu}{\sqrt{c_\mu}}\,d\sqrt T
    +
    \frac{k_\mu}{\sqrt{c_\mu}}\,dT\sqrt{1-\gamma}
    +
    \frac{\Gamma_T}{1-\gamma}
    \right).
    \]
    This proves the first bound.

    Now we prove the second bound.
    For the clipped tuned choice, let
    \[
    \gamma
    =
    1-
    \min\left\{
    1-\frac1T,\;
    \max\left\{
    \frac1T,\;
    \left(
    \frac{\Gamma_T\sqrt{c_\mu}}{k_\mu dT}
    \right)^{2/3}
    \right\}
    \right\}.
    \]
    By construction, we have \(\gamma\in[1/T,1)\).

    \paragraph{Case 1.} If
    $\frac{k_\mu d}{\sqrt{c_\mu T}}
    \le
    \Gamma_T
    \le
    \frac{k_\mu dT}{\sqrt{c_\mu}}
    \left(1-\frac1T\right)^{3/2},$
    then
    $\frac1T
    \le
    \left(
    \frac{\Gamma_T\sqrt{c_\mu}}{k_\mu dT}
    \right)^{2/3}
    \le
    1-\frac1T,$
    and therefore
    $1-\gamma
    =
    \left(
    \frac{\Gamma_T\sqrt{c_\mu}}{k_\mu dT}
    \right)^{2/3}.$
    This implies that
    \[
    \frac{k_\mu}{\sqrt{c_\mu}}\,dT\sqrt{1-\gamma}
    =
    \BigOTilde\!\left(
    \frac{k_\mu^{2/3}}{c_\mu^{1/3}}\,
    d^{2/3}\Gamma_T^{1/3}T^{2/3}
    \right),
    \]
    and
    \[
    \frac{\Gamma_T}{1-\gamma}
    =
    \BigOTilde\!\left(
    \frac{k_\mu^{2/3}}{c_\mu^{1/3}}\,
    d^{2/3}\Gamma_T^{1/3}T^{2/3}
    \right).
    \]
    Moreover, under
    $\Gamma_T\ge \frac{k_\mu d}{\sqrt{c_\mu T}},$
    the first term
    $\frac{k_\mu}{\sqrt{c_\mu}}\,d\sqrt T$
    is dominated by the same quantity.
    Therefore, we get
    \[
    \Regret
    =
    \BigOTilde\!\left(
    \frac{k_\mu^{2/3}}{c_\mu^{1/3}}\,
    d^{2/3}\Gamma_T^{1/3}T^{2/3}
    \right).
    \]

    \paragraph{Case 2.} If
    $0\le \Gamma_T<\frac{k_\mu d}{\sqrt{c_\mu T}},$
    then
    $\left(
    \frac{\Gamma_T\sqrt{c_\mu}}{k_\mu dT}
    \right)^{2/3}
    < \frac1T,$
    and hence
    $1-\gamma=\frac1T.$
    In this case, we have
    \[
    \frac{k_\mu}{\sqrt{c_\mu}}\,dT\sqrt{1-\gamma}
    =
    \BigOTilde\!\left(
    \frac{k_\mu}{\sqrt{c_\mu}}\,d\sqrt T
    \right),
    \]
    and
    \[
    \frac{\Gamma_T}{1-\gamma}
    =
    \Gamma_T T
    <
    \frac{k_\mu}{\sqrt{c_\mu}}\,d\sqrt T.
    \]
    Hence, we obtain
    \[
    \Regret
    \le
    \BigOTilde\!\left(
    \frac{k_\mu}{\sqrt{c_\mu}}\,d\sqrt T
    \right).
    \]

    \paragraph{Case 3.} If
    $\Gamma_T>
    \frac{k_\mu dT}{\sqrt{c_\mu}}
    \left(1-\frac1T\right)^{3/2},$
    then
    $\left(
    \frac{\Gamma_T\sqrt{c_\mu}}{k_\mu dT}
    \right)^{2/3}
    >
    1-\frac1T,$
    and thus
    $1-\gamma=1-\frac1T,$
    that is
    $\gamma=\frac1T.$
    In this case, we have
    \[
    \frac{k_\mu}{\sqrt{c_\mu}}\,dT\sqrt{1-\gamma}
    =
    \BigOTilde\!\left(
    \frac{k_\mu}{\sqrt{c_\mu}}\,dT
    \right),
    \]
    and
    \[
    \frac{\Gamma_T}{1-\gamma}
    =
    \BigOTilde\!\left(
    \Gamma_T
    \right).
    \]
    Moreover, the first term
    $\frac{k_\mu}{\sqrt{c_\mu}}\,d\sqrt T$
    is dominated by
    $\frac{k_\mu}{\sqrt{c_\mu}}\,dT.$
    Since \(T\ge 2\), we have
    $\left(1-\frac1T\right)^{3/2}\ge \frac{1}{2\sqrt2}.$
    Therefore,
    \[
    \Gamma_T>
    \frac{k_\mu dT}{\sqrt{c_\mu}}
    \left(1-\frac1T\right)^{3/2}
    \ge
    \frac{1}{2\sqrt2}\frac{k_\mu}{\sqrt{c_\mu}}\,dT,
    \]
    which implies
    \[
    \frac{k_\mu}{\sqrt{c_\mu}}\,dT
    \le
    2\sqrt2\,\Gamma_T.
    \]
    Hence, all three terms are dominated by
    \[
    \BigOTilde\!\left(
    \Gamma_T
    \right),
    \]
    and we obtain
    \[
    \Regret
    \le
    \BigOTilde\!\left(
    \Gamma_T
    \right).
    \]
    This completes the proof of Theorem~\ref{thm:regret_piecewise}.
\end{proof}

\subsection{Proofs of Lemmas for Theorem~\ref{thm:regret_piecewise}}
\label{app_subsec:proof_lemmas_thm:regret_piecewise}
\subsubsection{Proof of Lemma~\ref{lemma:pw_bad_rounds}}
\label{app_subsubsec:proof_lemma:pw_bad_rounds}
\begin{proof} [Proof of Lemma~\ref{lemma:pw_bad_rounds}]
    Let \(\nu_1,\ldots,\nu_{\Gamma_T}\) be the change points, that is, the rounds at which the parameter value changes.
    Fix any round \(t\in[T]\setminus \mathcal G_D\).
    By the definition of \(\mathcal G_D\), there exists
    $s\in \big\{\!\max\{t-D, 1\},\ldots,t-1 \big\}$
    such that
    \[
    \theta_s^\star \neq \theta_t^\star.
    \]
    Since the sequence \(\{\theta_u^\star\}_{u=1}^T\) is piecewise constant, the fact that
    \(\theta_s^\star \neq \theta_t^\star\) implies that the parameter must change at least once between rounds \(s\) and \(t\).
    Hence, there exists a change point \(\nu_j\), for some \(j\in[\Gamma_T]\), such that
    $\nu_j \in \{s+1,\ldots,t\}.$
    Because \(s \ge (t-D)\vee 1 = \max\{t-D,1\}\), we have
    $s+1 \ge \max\{t-D,1\}+1.$
    Therefore,
    $\nu_j \in \{s+1,\ldots,t\}
    \subseteq
    \{\max\{t-D,1\}+1,\ldots,t\}.$
    That is,
    \[
    \max\{t-D,1\}+1 \le \nu_j \le t.
    \]
    From
    $\nu_j \ge \max\{t-D,1\}+1,$
    we  obtain
    $\nu_j \ge t-D+1,$
    because \(\max\{t-D,1\} \ge t-D\).
    Rearranging this inequality gives
    \[
    t \le \nu_j + D - 1.
    \]
    Combining this with \(\nu_j \le t\), we obtain
    \[
    \nu_j \le t \le \nu_j + D - 1.
    \]
    Since also \(t\le T\), it follows that
    \[
    t \in \{\nu_j,\nu_j+1,\ldots,\min\{\nu_j+D-1,T\}\}.
    \]
    
    Because the choice of \(t \in [T]\setminus \mathcal G_D\) was arbitrary, we conclude that every bad round belongs to at least one such interval. 
    Therefore, we get
    \[
    [T]\setminus \mathcal G_D
    \subseteq
    \bigcup_{j=1}^{\Gamma_T}
    \{\nu_j,\nu_j+1,\ldots,\min\{\nu_j+D-1,T\}\}.
    \]
    
    For each \(j\in[\Gamma_T]\), the set
    $\{\nu_j,\nu_j+1,\ldots,\min\{\nu_j+D-1,T\}\}$
    contains at most \(D\) integers.
    Indeed, if \(\nu_j + D - 1 \le T\), then its cardinality is exactly
    $(\nu_j+D-1)-\nu_j+1 = D$,
    whereas if \(\nu_j + D - 1 > T\), then truncation at \(T\) only makes the set smaller.
    Hence,
    \[
    \Big|
    \{\nu_j,\nu_j+1,\ldots,\min\{\nu_j+D-1,T\}\}
    \Big|
    \le D
    \qquad
    \text{for all } j\in[\Gamma_T].
    \]
    
    Taking cardinalities and using the union bound for finite sets, we obtain
    \[
    |[T]\setminus \mathcal G_D|
    \le
    \sum_{j=1}^{\Gamma_T}
    \Big|
    \{\nu_j,\nu_j+1,\ldots,\min\{\nu_j+D-1,T\}\}
    \Big|
    \le
    \sum_{j=1}^{\Gamma_T} D
    =
    \Gamma_T D.
    \]
    This completes the proof.
\end{proof}
\subsubsection{Proof of Lemma~\ref{lemma:pw_kernel_sum}}
\label{app_subsubsec:proof_lemma:pw_kernel_sum}
\begin{proof} [Proof of Lemma~\ref{lemma:pw_kernel_sum}]
    Fix any \(t\in\mathcal G_D\cap\{2,\dots,T\}\). By definition,
    \begin{align*}
        \mathcal V_t
        &=
        \sum_{p=1}^{t-1}
        \gamma^{(t-1-p)/2}\sqrt{\frac{1-\gamma^p}{1-\gamma}}\,
        \|\theta_{p+1}^\star-\theta_p^\star\|_2
        \\
        &\le  \sum_{p=1}^{t-1}
        \frac{\gamma^{(t-1-p)/2}}{\sqrt{1-\gamma}}\,
        \|\theta_{p+1}^\star-\theta_p^\star\|_2
        \tag{\(1-\gamma^p\le 1\)}
        \\
        &\le
        \sum_{p=1}^{t-D-1}
        \frac{\gamma^{(t-1-p)/2}}{\sqrt{1-\gamma}}\,
        \|\theta_{p+1}^\star-\theta_p^\star\|_2,
    \end{align*}
    where the last inequality holds because, for \(t\in\mathcal G_D\),
    the parameter does not change over the most recent \(D\) rounds before \(t\).
    Here, the sum is interpreted as \(0\) if \(t-D-1<1\).
    
    Summing over \(t\in\mathcal G_D\cap\{2,\dots,T\}\) and then enlarging the index set to all \(t=2,\dots,T\), we get
    \begin{align*}
        \sum_{t\in\mathcal G_D\cap\{2,\dots,T\}}\mathcal V_t
        &\le
        \sum_{t=2}^{T}\sum_{p=1}^{t-D-1}
        \frac{\gamma^{(t-1-p)/2}}{\sqrt{1-\gamma}}\,
        \|\theta_{p+1}^\star-\theta_p^\star\|_2 \\
        &=
        \frac{1}{\sqrt{1-\gamma}}
        \sum_{p=1}^{T-1}
        \|\theta_{p+1}^\star-\theta_p^\star\|_2
        \sum_{t=p+D+1}^{T}\gamma^{(t-1-p)/2}
        \\
        &\le
        \frac{1}{\sqrt{1-\gamma}}
        \sum_{p=1}^{T-1}
        \|\theta_{p+1}^\star-\theta_p^\star\|_2
        \sum_{k=D}^{\infty}\gamma^{k/2}
        \\
        &=
        \frac{\gamma^{D/2}}{\sqrt{1-\gamma}(1-\sqrt\gamma)}
        \sum_{p=1}^{T-1}
        \|\theta_{p+1}^\star-\theta_p^\star\|_2
        =
        \frac{\gamma^{D/2}}{\sqrt{1-\gamma}(1-\sqrt\gamma)}\,P_T,
    \end{align*}
    For each fixed \(p\),
    where
    $P_T:=\sum_{p=1}^{T-1}\|\theta_{p+1}^\star-\theta_p^\star\|_2.$
    Using
    \[
    1-\sqrt\gamma=\frac{1-\gamma}{1+\sqrt\gamma}
    \qquad\text{and}\qquad
    1+\sqrt\gamma\le 2,
    \]
    we obtain
    \begin{equation}
        \sum_{t\in\mathcal G_D\cap\{2,\dots,T\}}\mathcal V_t
        \le
        \frac{2\gamma^{D/2}}{(1-\gamma)^{3/2}}\,P_T.
        \label{eq:lemma:pw_kernel_sum_intermid}
    \end{equation}
    It remains to bound \(P_T\) under piecewise stationarity.
    Let \(\nu_1,\dots,\nu_{\Gamma_T}\) be the change points. If
    \(p\notin\{\nu_1-1,\dots,\nu_{\Gamma_T}-1\}\), then
    \(\theta_{p+1}^\star=\theta_p^\star\), so
    \[
    \|\theta_{p+1}^\star-\theta_p^\star\|_2=0.
    \]
    On the other hand, if \(p=\nu_j-1\) for some \(j\in[\Gamma_T]\), then by the triangle inequality,
    \[
    \|\theta_{p+1}^\star-\theta_p^\star\|_2
    \le
    \|\theta_{p+1}^\star\|_2+\|\theta_p^\star\|_2
    \le 2S,
    \]
    since \(\theta_p^\star,\theta_{p+1}^\star\in\Theta\subseteq\{\theta:\|\theta\|_2\le S\}\).
    Because there are exactly \(\Gamma_T\) such indices,
    \[
    P_T
    =
    \sum_{p=1}^{T-1}\|\theta_{p+1}^\star-\theta_p^\star\|_2
    \le
    2S\Gamma_T.
    \]
    Substituting this into~\eqref{eq:lemma:pw_kernel_sum_intermid} yields
    \[
    \sum_{t\in\mathcal G_D\cap\{2,\dots,T\}}\mathcal V_t
    \le
    \frac{4S\Gamma_T\,\gamma^{D/2}}{(1-\gamma)^{3/2}}.
    \]
    This completes the proof.
\end{proof}

\end{document}